\documentclass[nonacm]{acmart}
\AtBeginDocument{%
  }

\usepackage{graphicx} 
\usepackage{multirow}
\usepackage{siunitx}
\usepackage{booktabs}
\usepackage{subcaption} 
\usepackage{algorithm}
\usepackage{algorithmic}
\newtheorem{Theorem}{Theorem}
\setcopyright{none}
\begin{document}
\settopmatter{printacmref=false}
\title{The Docking Game: Loop Self-Play for Fast, Dynamic, and Accurate Prediction of Flexible Protein--Ligand Binding}



\author{Youzhi Zhang}
\authornotemark[1]
\email{youzhi.zhang@cair-cas.org.hk}
\affiliation{%
  \institution{Centre for Artificial Intelligence and Robotics (CAIR), Hong Kong Institute of Science and Innovation (HKISI), Chinese Academy of Sciences}
  \country{Hong Kong SAR, China}
}

\author{Yufei Li}
\email{yufei_li19972@outlook.com}
\affiliation{%
  \institution{Centre for Artificial Intelligence and Robotics (CAIR), Hong Kong Institute of Science and Innovation (HKISI), Chinese Academy of Sciences}
  \country{Hong Kong SAR, China}
}

\author{Gaofeng Meng}
\authornote{Corresponding author.}
\email{gaofeng.meng@cair-cas.org.hk}
\affiliation{%
  \institution{Centre for Artificial Intelligence and Robotics (CAIR), Hong Kong Institute of Science and Innovation (HKISI), Chinese Academy of Sciences}
  \country{Hong Kong SAR, China}
}

\author{Hongbin Liu}
\email{hongbin.liu@cair-cas.org.hk}
\affiliation{%
  \institution{Centre for Artificial Intelligence and Robotics (CAIR), Hong Kong Institute of Science and Innovation (HKISI), Chinese Academy of Sciences}
  \country{Hong Kong SAR, China}
}

\author{Jiebo Luo}
\email{jluo@hkisi.org.hk}
\affiliation{%
  \institution{Hong Kong Institute of Science and Innovation (HKISI), Chinese Academy of Sciences}
 \country{Hong Kong SAR, China}
}



\begin{abstract}
Molecular docking is a crucial aspect of drug discovery, as it predicts the binding interactions between small-molecule ligands and protein pockets. However, current multi-task learning models for docking often show inferior performance in ligand docking compared to protein pocket docking. This disparity arises largely due to the distinct structural complexities of ligands and proteins. 
To address this issue, we propose a novel game-theoretic framework that models the protein-ligand interaction as a two-player game called \textit{the Docking Game}, with the ligand docking module acting as the ligand player and the protein pocket docking module as the protein player. To solve this game, we develop a novel Loop Self-Play (\textbf{LoopPlay}) algorithm, which alternately trains these players through a two-level loop.
In the outer loop, the players exchange predicted poses, allowing each to incorporate the other's structural predictions, which fosters mutual adaptation over multiple iterations. In the inner loop, each player dynamically refines its predictions by incorporating its own predicted ligand or pocket poses back into its model.
We theoretically show the convergence of LoopPlay, ensuring stable optimization. Extensive experiments conducted on public benchmark datasets demonstrate that LoopPlay achieves approximately a 10\% improvement in predicting accurate binding modes compared to previous state-of-the-art methods. This highlights its potential to enhance the accuracy of molecular docking in drug discovery.
\end{abstract}
 




\maketitle

\section{Introduction}

Small molecules—organic compounds with low molecular weight—are the primary therapeutic agents in the pharmaceutical industry \cite{Zhang2023}. These compounds exert their effects by binding to target proteins and modulating molecular pathways associated with diseases. Understanding the structural details of protein-ligand interactions is crucial for determining drug potency, mechanisms of action, and potential side effects.
Despite significant efforts to elucidate the structures of protein-ligand complexes, the Protein Data Bank (PDB) \cite{berman2000protein} contains only a small fraction of the possible configurations. This scarcity is particularly striking when considering the immense combinatorial space of potential interactions, which includes over $10^{60}$ drug-like molecules \cite{hert2009quantifying, luttens2025rapid} and at least 20,000 human proteins \cite{uniprot2019uniprot}. This gap highlights the urgent need for reliable protein-ligand docking methods \cite{Zhang2023}.

Molecular docking is a fundamental aspect of drug discovery, focused on predicting the binding structures of protein-ligand complexes \cite{Morris1996,agarwal2016overview}. This helps us understand how small drug-like molecules, known as ligands, interact with target proteins \cite{Morris1996, agarwal2016overview}. Traditional methods, based on principles from physics and chemistry, often come with high computational costs \cite{Friesner2004, Trott2010, McNutt2021}. However, the emergence of deep learning-based docking methods \cite{Corso2023, Zhang2023, Pei2023, gao2025fabind+, zhang2025fast} has brought about a significant shift, allowing for faster and more accurate predictions of docked protein-ligand poses.
Many of these new methods \cite{Lu2022, Zhang2023, Stark2022, Pei2023, gao2025fabind+, Corso2023, zhou2023unimol} operate under the rigid docking paradigm, which assumes that proteins remain static during the docking process. This assumption does not align with the dynamic and flexible nature of proteins in physiological conditions \cite{HenzlerWildman2007, Lane2023}, thus limiting the practical applicability of these approaches. Consequently, it is crucial to develop flexible docking methods that better account for realistic protein dynamics \cite{Sahu2024}.

Recent flexible docking methods \cite{Lu2024,Huang2024,zhang2024packdock,Qiao2024} primarily utilize diffusion models and sampling strategies \cite{Yang2023} because of their strong ability to model complex distributions. However, the iterative nature of the diffusion process and the extensive sampling can decrease computational efficiency \cite{zhang2025fast}. On the other hand, regression-based methods \cite{Lu2022,Zhang2023,Stark2022,Pei2023} provide faster predictions by directly estimating the bound structures of protein-ligand complexes using specialized neural networks. The regression-based FABind series \cite{Pei2023,gao2025fabind+} aims to strike a balance between docking accuracy and computational efficiency through an end-to-end multitask model. This model integrates pocket prediction and docking prediction, thereby eliminating the need for a separate external module for pocket selection. FABFlex \cite{zhang2025fast} extends the FABind series to flexible docking scenarios by incorporating a pocket docking prediction module.

However, the regression-based multi-task framework for protein-ligand docking, as demonstrated by FABFlex \cite{zhang2025fast}, shows a significant weakness in its ligand docking performance compared to protein pocket docking. This discrepancy arises primarily due to the differing properties of ligands and protein pockets. Although FABFlex employs separate modules for ligand docking (predicting ligand holo structures from apo states) and protein pocket docking (forecasting holo conformations of protein pockets), the joint training through a multi-task learning model leads to suboptimal outcomes for ligand docking. The average Root Mean Square Deviation (RMSD) for ligand docking is 5.44, while that for pocket docking is only 1.1.
This performance gap is likely a result of the substantial differences in structural complexity between the two tasks. For instance, protein pockets typically have many more nodes than ligands, with pockets like the one in PDB 3ZZF containing 144 nodes compared to just 13 for ligands in the same structure. These differences can lead the model to prioritize the more complex pocket docking task, resulting in inadequate optimization for the ligand docking module.

\begin{figure}
    \centering
    \includegraphics[width=0.9\linewidth]{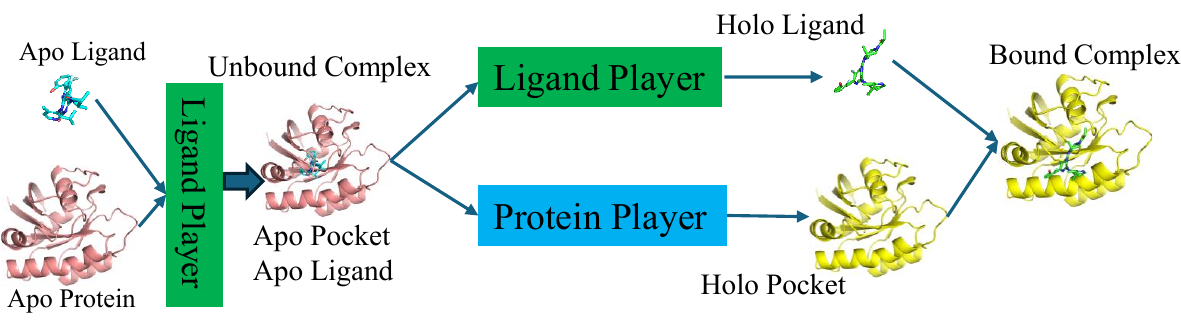}
    \caption{\textit{The Docking Game}: given apo ligand and protein structures, the ligand player first identifies the protein pocket location and repositions the ligand from the protein's center to the pocket's center, forming an unbound protein pocket--ligand complex. Subsequently, the ligand player and protein player independently predict the holo ligand pose and holo protein pocket pose, respectively, to construct the bound protein pocket--ligand complex.}
    \label{fig:ligand_protein_game}
\end{figure}
 
 To enhance protein-ligand docking performance, we propose a novel game-theoretic framework that models the interaction between end-to-end ligand docking pose prediction and protein pocket docking pose prediction as a two-player game called   \textit{the Docking Game}, with the ligand docking module as the \textit{ligand player} and the protein pocket docking module as the \textit{protein player}. Drawing inspiration from the FABind series \cite{Pei2023,gao2025fabind+,zhang2025fast}, the ligand player incorporates an integrated pocket prediction module to eliminate the need for an external pocket selection step. In our game-theoretic approach to blind flexible protein-ligand docking, given apo ligand and protein structures, the ligand player first identifies the protein pocket location and repositions the ligand from the protein's center to the pocket's center, forming an unbound protein pocket--ligand complex. Subsequently, the ligand player and protein player independently predict the holo ligand pose and holo protein pocket pose, respectively, to construct the bound protein pocket--ligand complex, as illustrated in Figure \ref{fig:ligand_protein_game}. 
 
  To solve this game, we develop a novel Loop Self-Play (\textbf{LoopPlay}) algorithm to alternately train these two models. In LoopPlay, each player model undergoes training through a two-level loop.
  In the outer loop, the ligand and protein players exchange predicted poses to incorporate the opponent's structural insights, fostering mutual adaptation over multiple iterations. In the inner loop, each player dynamically refines its predictions by iteratively feeding its own predicted ligand or pocket poses back into its model, enhancing docking pose accuracy.
 
  Our key contributions are the following:
  \begin{itemize}
      \item We propose a novel two-player game model, called \textit{the Docking Game}, to represent and enforce the interaction between ligand docking and protein pocket docking modules. 
      \item We develop a novel Loop Self-Play (LoopPlay) algorithm with theoretical convergence, to ensure stable
optimization. 
      \item We conduct extensive experiments on public benchmark datasets and show that LoopPlay significantly improves the prediction of accurate binding modes, achieving approximately a 10\% improvement compared to existing state-of-the-art methods. 
  \end{itemize}

\section{Related Work}

\textbf{Flexible Molecular Docking.} DiffPack \cite{zhang2023diffpack} is a torsional diffusion model that captures the joint distribution of side-chain torsional angles, the sole degrees of freedom in side-chain packing with ligands, by performing diffusion and denoising in the torsional space. To mitigate challenges from simultaneously perturbing all four torsional angles, it autoregressively generates these angles and trains distinct diffusion models for each one.   DIFFDOCK-POCKET \cite{plainer2023diffdock} is a diffusion-based docking method that predicts ligand poses within a designated binding pocket by conditioning on the target protein. Additionally, it accounts for receptor flexibility and determines the positions of sidechains near the binding site.
  Re-Dock \cite{Huang2024}   presents a diffusion bridge generative model for flexible molecular docking, concurrently predicting ligand and pocket sidechain poses by applying diffusion processes to geometric manifolds. It utilizes an energy-to-geometry mapping, inspired by the Newton-Euler equation, to explicitly capture protein-ligand interactions, producing realistic and physically plausible conformations. DynamicBind \cite{Lu2024} is a deep learning approach that uses equivariant geometric diffusion networks to create a smooth energy landscape, facilitating efficient shifts between various states. It predicts ligand-specific conformations from unbound protein structures, eliminating the need for holo-structures.   NeuralPLexer \cite{Qiao2024} is a computational method that predicts protein-ligand complex structures directly from protein sequences and ligand molecular graphs. It employs a deep generative model to generate three-dimensional binding complex structures and their conformational changes at atomic resolution. The model uses a diffusion process with key biophysical constraints and a multi-scale geometric deep learning system to iteratively sample all heavy-atom coordinates and residue-level contact maps in a hierarchical fashion. PackDock \cite{zhang2024packdock} is a flexible docking approach that integrates "induced fit" and "conformation selection"  mechanisms within a two-stage docking process. Its central component, PackPocket, employs a diffusion model to navigate the side-chain conformation space in ligand binding pockets, effective in both ligand-present and ligand-absent scenarios. Although these methods improve docking performance, they suffer from the inherent inefficiencies of multi-round sampling and diffusion models, resulting in low computational efficiency. This drawback limits their scalability for evaluating large-scale, unknown protein-ligand interactions critical for drug discovery. FABFlex \cite{zhang2025fast} is a regression-based multi-task learning model optimized for rapid and precise blind flexible docking in realistic scenarios where protein flexibility and unknown binding pocket locations pose significant challenges. Its architecture integrates three synergistic modules: (1) a pocket prediction module that identifies potential binding sites to tackle the complexities of blind docking, (2) a ligand docking module that predicts ligand holo structures from their apo states, and (3) a pocket docking module that forecasts the holo conformations of protein pockets starting from their apo forms. 
 Instead of training modules together via a multi-task learning model, we adopt a game-theoretic approach to alternately train the ligand docking module and protein pocket docking module, which significantly improves the ligand docking performance in blind flexible docking problems.   
 
\textbf{Self-Play.}  
In the realm of game-theoretic approaches \cite{shoham2008multiagent}, self-play has been extensively explored as a powerful paradigm for optimizing strategies in competitive settings \cite{silver2016mastering}, with applications ranging from game theory to machine learning \cite{goodfellow2014generative}. Self-play in game theory involves agents iteratively refining their strategies by playing against themselves, converging toward equilibrium solutions such as Nash equilibria, as seen in seminal works on multi-agent systems \cite{shoham2008multiagent}. Notable self-play algorithm frameworks include Policy-Space Response Oracles (PSRO) \cite{lanctot2017unified} and Fictitious Play \cite{brown:fp1951}, which have advanced the field by modeling strategic interactions. PSRO iteratively constructs a meta-game where agents learn the best responses to a mixture of opponent policies, demonstrating successes in complex games such as poker. Fictitious Play, on the other hand, assumes agents update their strategies based on the empirical distribution of past opponent actions, converging to equilibrium in certain game classes. Additionally, Generative Adversarial Networks (GANs) \cite{goodfellow2014generative} embody a self-play-like framework in machine learning, where a generator and discriminator compete in a minimax game, optimizing through adversarial training to produce realistic data distributions. These approaches share similarities with our LoopPlay algorithm, which leverages a game-theoretic self-play framework to model ligand and protein pocket docking as interacting players, but distinguishes itself through its novel two-level loop structure for iterative refinement and cross-module pose exchange, achieving superior performance in protein-ligand docking tasks.

\section{Preliminaries}

\subsection{Blind Flexible Docking} Blind flexible molecular docking aims to predict the bound structure of a protein pocket--ligand complex using unbound (apo) conformations: a randomly initialized apo ligand generated by RDKit \cite{Landrum2013} and an AlphaFold2-predicted apo protein \cite{Jumper2021}. The objective is to produce the coordinates of the docked ligand (holo ligand) and the docked binding protein pocket (holo pocket), denoted as $\mathbf{\hat{x}} = \{\{\mathbf{\hat{x}}_i\}_{i=1}^{n_l},  \{\mathbf{\hat{x}}_j\}_{j=1}^{n_{p^*}}\}$, referred to as the bound complex. Here,  $n_l = |V_l|$ and $n_{p^*} = |V_{p^*}|$ represent the number of ligand atoms and pocket residues, respectively. To do that, we model the protein--ligand complex as a graph.

\textbf{Protein--Ligand Complex:}
Each protein--ligand complex is represented as a heterogeneous graph:
\[
G = \{V := (V_L, V_P), E := (E_L, E_P, E_{LP})\},
\]
where $V$ and $E$ are the sets of nodes and edges, respectively. This graph includes a ligand subgraph and a protein subgraph:

\begin{itemize}
  \item Ligand subgraph: $G_L = \{V_L, E_L\}$, where each edge in $ E_L$ represents a chemical bond of the ligand, and  each node $v_i = (\mathbf{h}_i, \mathbf{x}_i) \in V_L$ represents an atom in the ligand with:
    \begin{itemize}
      \item $\mathbf{h}_i \in \mathbb{R}^{d_l}$ is the feature vector extracted using TorchDrug \cite{zhu2022torchdrug},
      \item $\mathbf{x}_i \in \mathbb{R}^3$ is the 3D spatial coordinate of   the atom.
    \end{itemize}
  \item Protein subgraph: $G_P = \{V_P, E_P\}$, where each edge in $E_P$ connects two residues within an 8\AA{} radius,  each node $v_j = (\mathbf{h}_j, \mathbf{x}_j) \in V_P$ represents a residue in  the protein with:
    \begin{itemize}
      \item $\mathbf{h}_j \in \mathbb{R}^{d_p}$ is the feature derived from ESM-2 \cite{Lin2022},
      \item $\mathbf{x}_j \in \mathbb{R}^3$ is the 3D coordinate of the C$\alpha$ atom of the residue.
    \end{itemize}
  \item External interface edges $E_{LP}$ connect nodes $v_i \in V_L$ and $v_j \in V_P$ if they are spatially within $ 10$\AA{}.
\end{itemize}

For the true protein pocket region, we define a reduced heterogeneous graph to represent the protein pocket--ligand complex:
\[
G^* = \{V^* := (V_L, V_P^*), E^* := (E_L, E_P^*, E_{LP}^*)\},
\]
where $V_P^*$, $E_P^*$, and $E_{LP}^*$ represent the subset of the protein comprising the pocket. Similarly, the predicted protein pocket region is denoted using hat notation:
\[
\hat{G}^* = \{\hat{V}^* := (V_L, \hat{V}_P^*), \hat{E}^* := (E_L, \hat{E}_P^*, \hat{E}_{LP}^*)\}.
\]

 \subsection{The FABind Layer}
The foundational unit shared by the regression-based method FABind \cite{Pei2023}  and its advanced versions FABind+ \cite{gao2025fabind+} and FABFlex \cite{zhang2025fast} is the FABind layer—an improved E(3)-equivariant graph neural network (EGNN) \cite{Satorras2021} designed for protein--ligand heterogeneous graphs. Multiple FABind layers can be stacked to deepen the model and capture higher-order relational features, improving prediction accuracy across ligand–protein docking tasks.
At the $l$-th FABind   layer, node features, spatial coordinates, and pairwise embeddings are updated as:
\[
\mathbf{h}^{(l+1)}_i,\ \mathbf{h}^{(l+1)}_j,\ \mathbf{x}^{(l+1)}_i,\ \mathbf{x}^{(l+1)}_j,\ \textbf{p}^{(l+1)}_{ij} = f(\mathbf{h}^{(l)}_i,\ \mathbf{h}^{(l)}_j,\ \mathbf{x}^{(l)}_i,\ \mathbf{x}^{(l)}_j,\ \textbf{p}^{(l)}_{ij}),
\]
where $\textbf{p}_{ij} \in \mathbb{R}^d$ represents the interaction embedding between a ligand atom $v_i \in V_l$ and a protein residue $v_j \in V_p$, with hidden dimension of $d$.

\section{A Game-Theoretic Approach to Docking}
We propose a game-theoretic model for blind flexible docking, which is formulated as a two-player game called    \textit{the Docking Game} between a \textit{Ligand Player} and a \textit{Protein Player}.  Each player addresses distinct subtasks to predict the bound protein pocket--ligand complex, as shown in Figure  \ref{fig:ligand_protein_game}. The ligand player handles two subtasks: identifying pocket sites, modeled as a binary classification problem to determine which protein residues form the docking pocket, and predicting the bound holo structure of the ligand, treated as a 3D coordinate regression problem. The protein player focuses on predicting the bound holo structure of the protein pocket, also formulated as a 3D coordinate regression problem. Formally, we have:  
\begin{itemize}
    \item \textbf{Players: Ligand Player (L)   and Protein Player (P)}\\
    In the game-theoretic model for blind flexible docking, the problem is conceptualized as a cooperative game involving two players: the ligand player (L) and the protein player (P). Each player represents a computational agent responsible for specific subtasks in the docking process. The ligand player handles tasks related to the ligand, including identifying pocket sites and then predicting the ligand's bound holo structure based on it. The protein player focuses on predicting the bound holo structure of the protein pocket.  
    \item \textbf{Strategies: $\theta_L \in \Theta_L$, $\theta_P \in \Theta_P$, where $\Theta_L$, $\Theta_P$ are parameter spaces} \\
    The ligand player is parameterized by $\theta_L \in 
    \Theta_L$, and the protein player is  parameterized by $\theta_P\in \Theta_P$.
    The strategies in this game are the choices of parameters $\theta_L$ and $\theta_P$ for the ligand and protein players, respectively. These parameters belong to their respective parameter spaces, $\Theta_L$ and $\Theta_P$, which are sets of all possible parameter configurations for the models used by each player. For example, $\Theta_L$ might represent the space of weights for a neural network predicting pocket sites and ligand coordinates, while $\Theta_P$ represents the weights for a network predicting pocket coordinates.  

    \item \textbf{Payoffs: Define the losses as cost functions to minimize} \\
    The payoffs are defined as loss functions $\mathcal{J}_L(\theta_L, \theta_P)$ for the ligand player and $\mathcal{J}_P(\theta_L, \theta_P)$ for the protein player, which \textit{both} players aim to minimize together. These functions quantify the performance of the players' strategies in achieving accurate docking predictions, balancing individual subtask losses with a shared interaction term. The   components are (which  will be defined formally later): 
    \begin{itemize}
        \item \textbf{For the Ligand Player:} $\mathcal{J}_L(\theta_L, \theta_P) = \alpha_1 \mathcal{L}_{\text{pocket\_pred}}(\theta_L) + \alpha_2 \mathcal{L}_{\text{ligand\_coord}}(\theta_L) + \gamma \mathcal{L}_{\text{dis\_map}}(\theta_L, \theta_P)$ \\
        The loss function for the ligand player consists of three terms:
        \begin{itemize}
            \item $\mathcal{L}_{\text{pocket\_pred}}(\theta_L)$: This term represents the loss for the binary classification subtask of identifying pocket sites, determining which protein residues form the binding pocket. It is a function of $\theta_L$, as the ligand player is responsible for this prediction. 
            \item $\mathcal{L}_{\text{ligand\_coord}}(\theta_L)$: This term captures the loss for the 3D coordinate regression subtask of predicting the bound holo structure of the ligand. It measures the discrepancy between the predicted and true 3D coordinates of the ligand atoms, depending only on $\theta_L$.
            \item $\mathcal{L}_{\text{dis\_map}}(\theta_L, \theta_P)$: This shared term represents the distance map loss, which ensures the compatibility between the ligand and protein predictions. It quantifies the agreement between the predicted ligand and pocket structures by comparing pairwise distances, and depends on both $\theta_L$ and $\theta_P$.
            \item $\alpha_1, \alpha_2, \gamma$: These are weighting coefficients that balance the contributions of each loss term. $\alpha_1$ and $\alpha_2$ prioritize the pocket prediction and ligand coordinate tasks, respectively, while $\gamma$ controls the influence of the shared distance map loss, ensuring the ligand and protein predictions align.
        \end{itemize}
        \item \textbf{For the Protein Player:} $\mathcal{J}_P(\theta_L, \theta_P) = \beta \mathcal{L}_{\text{pocket\_coord}}(\theta_P) + \gamma \mathcal{L}_{\text{dis\_map}}(\theta_L, \theta_P)$ \\
        The loss function for the protein player includes two terms:
        \begin{itemize}
            \item $\mathcal{L}_{\text{pocket\_coord}}(\theta_P)$: This term represents the loss for the 3D coordinate regression subtask of predicting the bound holo structure of the pocket. It measures the error between the predicted and true 3D coordinates of the pocket residues, depending only on $\theta_P$.
            \item $\mathcal{L}_{\text{dis\_map}}(\theta_L, \theta_P)$: The same term in the loss function of the ligand player. 
            \item $\beta, \gamma$: These coefficients weight the pocket coordinate loss and the shared distance map loss, respectively. $\beta$ prioritizes the accuracy of the protein pocket structure prediction, while $\gamma$ (shared with the ligand player's loss) ensures alignment between the ligand and protein predictions.
        \end{itemize}
    \end{itemize}

    \item \textbf{Nash Equilibrium:} The solution concept for \textit{the Docking Game} \\ 
    The Nash equilibrium represents an optimal solution where neither player can improve their payoff (i.e., reduce their loss) by unilaterally changing their strategy, given the other player's strategy. Mathematically, $(\theta_L^*, \theta_P^*)$ is a Nash equilibrium if:
    \[
    \mathcal{J}_L(\theta_L^*, \theta_P^*) \leq \mathcal{J}_L(\theta_L, \theta_P^*) \quad \forall \theta_L \in \Theta_L,
    \]
    \[
    \mathcal{J}_P(\theta_L^*, \theta_P^*) \leq \mathcal{J}_P(\theta_L^*, \theta_P) \quad \forall \theta_P \in \Theta_P.
    \]
    This means:
    \begin{align*}
        & \theta_L^* = \arg\min_{\theta_L} \left[ \alpha_1 \mathcal{L}_{\text{pocket\_pred}}(\theta_L) + \alpha_2 \mathcal{L}_{\text{ligand\_coord}}(\theta_L) + \gamma \mathcal{L}_{\text{dis\_map}}(\theta_L, \theta_P^*) \right],\\
        &\theta_P^* = \arg\min_{\theta_P} \left[ \beta \mathcal{L}_{\text{pocket\_coord}}(\theta_P) + \gamma \mathcal{L}_{\text{dis\_map}}(\theta_L^*, \theta_P) \right].
    \end{align*}
    In this context, the Nash equilibrium corresponds to the optimal parameters $\theta_L^*$ and $\theta_P^*$, where the ligand player has minimized its loss for pocket site identification and ligand coordinate prediction, and the protein player has minimized its loss for pocket coordinate prediction, while both ensure their predictions are consistent via the shared distance map loss. The equilibrium ensures that the predicted ligand and pocket structures are mutually optimal and compatible, solving the blind flexible docking problem effectively.
\end{itemize}

\subsection{The Task and Loss Function of the Ligand Player}

To tackle the complexities of blind flexible docking, the   ligand player has two tasks, and each task is represented  by a module of neural networks: 
\begin{itemize}
    \item \textbf{Pocket Prediction Module} $\mathcal{M}_S(\cdot)$: Determines which protein residues constitute the binding pocket.
    \item \textbf{Ligand Docking Module} $\mathcal{M}_L(\cdot)$: Predicts the bound structure of the ligand.
\end{itemize}

Each module is built using stacked FABind layers, a form of improved E(3)-equivariant graph neural networks optimized for heterogeneous graphs representing ligand–protein interactions.
The pocket prediction module first processes the protein--ligand graph $G$ to output an indicator vector $\{\hat{y}_j\}_{j=1}^{n_p}$:
\[
\{\hat{y}_j\}_{j=1}^{n_p} = \mathcal{M}_S(G, \{\mathbf{h}_i, \mathbf{x}_i\}_{i=1}^{n_l}, \{\mathbf{h}_j, \mathbf{x}_j\}_{j=1}^{n_p}),
\quad
\hat{V}_{p^*} = \{\hat{y}_j \odot v_j\}_{j=1}^{n_p}
\]
where $\hat{y}_j \in \{0,1\}$ denotes whether residue $v_j$ is part of the pocket, and $\odot$ represents the selection operation.

Given the rarity of pocket-forming residues within a protein, the task of detecting pocket residues is approached as an imbalanced binary classification problem, utilizing binary cross-entropy loss (BCELoss) to identify which residues constitute the pocket. The loss is defined as:
\[
\mathcal{L}_{\text{pocket\_cls}} = \frac{1}{n} \sum_{i=1}^n \frac{p_i}{q_i} \left\{ -\sum_{j=1}^{p_i} \left[ y_i \log \left( \hat{y}_i \right) + \left( 1 - y_i \right) \log \left( 1 - \hat{y}_i \right) \right] \right\},
\]
where $n$ denotes the number of training complexes, $p_i$ represents the total residues in the $i$-th protein, and $q_i$ indicates the number of pocket-forming residues. The weighting factor $p_i / q_i$ enhances the focus on proteins with fewer pocket residues, ensuring their significance during training. Additionally, the pocket center loss, $\mathcal{L}_{\text{pocket\_center}}$, employs the Huber loss \cite{huber1992robust} to guide the prediction of the pocket's centroid, expressed as:
\[
\mathcal{L}_{\text{pocket\_center}} = \frac{1}{n} \sum_{i=1}^n \text{HuberLoss} \left( \text{center}_i, \widehat{\text{center}}_i \right),
\]
where $\text{center}_i$ is the true centroid of the pocket in the $i$-th protein, and $\widehat{\text{center}}_i$ is the predicted centroid, computed as a weighted average of residue coordinates using probabilities from the Gumbel-Softmax distribution \cite{jang2017categorical}. The overall pocket prediction loss, $\mathcal{L}_{\text{pocket\_pred}}$, integrates the residue classification loss $\mathcal{L}_{\text{pocket\_cls}}$ and the pocket center loss $\mathcal{L}_{\text{pocket\_center}}$:
\[
\mathcal{L}_{\text{pocket\_pred}} = \alpha^{cls}_1 \mathcal{L}_{\text{pocket\_cls}} + \alpha^{center}_1 \mathcal{L}_{\text{pocket\_center}}.
\]

Using the predicted pocket set $\hat{V}_{p^*}$, the ligand  docking module refinse the protein--ligand graph into a protein pocket--ligand subgraph $\hat{G}^*$ and generate the holo structure of the ligand:
\[
\{\mathbf{\mathbf{\hat{x}}}_i\}_{i=1}^{n_l} = \mathcal{M}_L(\hat{G}^*, \{\mathbf{h}_i, \mathbf{x}_i\}_{i=1}^{n_l}, \{\mathbf{h}_j, \mathbf{x}_j\}_{j=1}^{\hat{n}_{p^*}}). 
\]

 For the ligand coordinate predictions, the Huber loss is used to align the predicted holo structures with their true counterparts:
\[
\mathcal{L}_{\text{ligand\_coord}} = \frac{1}{n} \sum_{i=1}^n \text{HuberLoss} \left( \mathbf{x}_i^l, \hat{\mathbf{x}}_i^l \right),  
\]
where $\mathbf{x}_i^l$ is the true holo coordinate for the ligand, and $\hat{\mathbf{x}}_i^l$ is the   predicted coordinate for the $i$-th complex.

The distance map loss, $\mathcal{L}_{\text{distance}}$, acts as a supplementary objective, using mean-squared-error loss (MSELoss) to regulate the relative positions between ligand atoms and pocket residues, defined as:
\[
\mathcal{L}_{\text{dis\_map}} = \frac{1}{n} \sum_{i=1}^n \text{MSELoss} \left( \mathbf{D}_i, \hat{\mathbf{D}}_i \right),
\]
where $\mathbf{D}_i$ and $\hat{\mathbf{D}}_i$ are the true and predicted distance maps, respectively, with $\mathbf{D}_i^{jk}$ representing the Euclidean distance between the $j$-th ligand atom and the $k$-th pocket residue, and $\hat{\mathbf{D}}_i^{jk}$ denoting the predicted pairwise distances.

Finally, we have the loss function for the ligand player: $\mathcal{J}_L(\theta_L, \theta_P) = \alpha_1 \mathcal{L}_{\text{pocket\_pred}}(\theta_L) + \alpha_2 \mathcal{L}_{\text{ligand\_coord}}(\theta_L) + \gamma \mathcal{L}_{\text{dis\_map}}(\theta_L, \theta_P)$. This loss function is inspired by FABind \cite{Pei2023},    FABind+ \cite{gao2025fabind+}, and FABFlex \cite{zhang2025fast}.
 
\subsection{The Task and Loss Function of the Protein Player}

To tackle the complexities of blind flexible docking, the protein player has one task,  which is represented by a module of neural networks called \textbf{Pocket Docking Module} $\mathcal{M}_P(\cdot)$: predicting the bound structure of the protein pocket. This module is built using stacked FABind layers, a form of improved E(3)-equivariant graph neural networks optimized for heterogeneous graphs representing ligand–protein interactions. 

Using the predicted pocket set $\hat{V}_{p^*}$, the   pocket docking module refine the protein--ligand graph into a protein pocket--ligand subgraph $\hat{G}^*$ and generate the holo structure of the protein pocket:
\[
\{\mathbf{\hat{x}}_j\}_{j=1}^{\hat{n}_{p^*}} = \mathcal{M}_P(\hat{G}^*, \{\mathbf{h}_i, \mathbf{x}_i\}_{i=1}^{n_l}, \{\mathbf{h}_j, \mathbf{x}_j\}_{j=1}^{\hat{n}_{p^*}}).
\]

 For the pocket coordinate predictions, the Huber loss is used to align the predicted holo structures with their true counterparts:
\[
\mathcal{L}_{\text{pocket\_coord}} = \frac{1}{N} \sum_{i=1}^N \text{HuberLoss} \left( \mathbf{x}_i^p, \hat{\mathbf{x}}_i^p \right),
\]
where   $\mathbf{x}_i^p$ is the true holo coordinate for the   pocket,  and   $\hat{\mathbf{x}}_i^p$ is the   predicted coordinate for the $i$-th complex.

Finally, we have the loss function for the protein player:   $\mathcal{J}_P(\theta_L, \theta_P) = \beta \mathcal{L}_{\text{pocket\_coord}}(\theta_P) + \gamma \mathcal{L}_{\text{dis\_map}}(\theta_L, \theta_P)$. This loss function is inspired by FABFlex \cite{zhang2025fast}.


\begin{algorithm}[H]
\caption{LoopPlay}\label{alg_loopplay}
\begin{algorithmic}[1]
\STATE \textbf{Input:} Initialized models $\mathcal{M}_S$,  $\mathcal{M}_L$, and  $\mathcal{M}_P$,  training dataset $\mathcal{D}$ with the 3D coordinates of apo protein and ligand with features represented by  $\{\mathbf{h}_i, \mathbf{x}_i\}_{i=1}^{n_l}$ and $\{\mathbf{h}_j, \mathbf{x}_j\}_{j=1}^{n_p}$   

\STATE ActingPlayer $\gets$ Player L \hfill // The player model that is trained
\FOR{each epoch and each complex}
\STATE Place ligand at the center of the protein and construct protein--ligand graph $G$
\STATE  
$\{\hat{y}_j\}_{j=1}^{n_p} \gets \mathcal{M}_S(G, \{\mathbf{h}_i, \mathbf{x}_i\}_{i=1}^{n_l}, \{\mathbf{h}_j, \mathbf{x}_j\}_{j=1}^{n_p})$\hfill // Pocket prediction
\STATE $\hat{V}_{p^*} \gets \{\hat{y}_j \odot v_j\}_{j=1}^{n_p}
$ \hfill // Select pocket residues
    \IF{ActingPlayer = Player L}
        \FOR{$k \gets 1$ to $M_L$}
            \STATE Construct protein pocket--ligand graph $\hat{G}^*$ with $\{\mathbf{\hat{x}}^{k,1}_i\}_{i=1}^{n_l} \gets \{\mathbf{\hat{x}}^{k}_i\}_{i=1}^{n_l}
           $
            \FOR{$r \gets 1$ to $N_L$}
                \STATE $\{\mathbf{\hat{x}}^{k,r+1}_i\}_{i=1}^{n_l} \gets \mathcal{M}_L(\hat{G}^*, \{\mathbf{h}_i, \mathbf{\hat{x}}^{k,r}_i\}_{i=1}^{n_l}, \{\mathbf{h}_j, \mathbf{\hat{x}}^k_j\}_{j=1}^{\hat{n}_{p^*}})$
            \ENDFOR
            \STATE  $\{\mathbf{\hat{x}}^{k+1}_i\}_{i=1}^{n_l} \gets \{\mathbf{\hat{x}}^{k,N_L}_i\}_{i=1}^{n_l}
           $\hfill // Ligand docking
            \STATE  $\{\mathbf{\hat{x}}^{k+1}_j\}_{j=1}^{\hat{n}_{p^*}} \gets $   Protein Model with the input $(\hat{G}^*, \{\mathbf{h}_i, \mathbf{\hat{x}}^k_i\}_{i=1}^{n_l}, \{\mathbf{h}_j, \mathbf{\hat{x}}^k_j\}_{j=1}^{\hat{n}_{p^*}})$  \hfill // Protein Player
        \ENDFOR
        \STATE Update models $\mathcal{M}_S$ and $\mathcal{M}_L$ based on the loss function $\mathcal{J}_L$
    \ELSE
        \FOR{$k \gets 1$ to $M_P$}
            \STATE Construct protein pocket--ligand graph $\hat{G}^*$ with $\{\mathbf{\hat{x}}^{k,1}_j\}_{j=1}^{\hat{n}_{p^*}} \gets \{\mathbf{\hat{x}}^{k}_j\}_{j=1}^{n_p^*}
           $
            \FOR{$r \gets 1$ to $N_P$}
                \STATE $\{\mathbf{\hat{x}}^{k,r+1}_j\}_{j=1}^{\hat{n}_{p^*}} \gets \mathcal{M}_P(\hat{G}^*, \{\mathbf{h}_i, \mathbf{\hat{x}}^k_i\}_{i=1}^{n_l}, \{\mathbf{h}_j, \mathbf{x}^{k,r}_j\}_{j=1}^{\hat{n}_{p^*}})$
            \ENDFOR
            \STATE  $\{\mathbf{\hat{x}}^{k+1}_j\}_{j=1}^{\hat{n}_{p^*}} \gets \{\mathbf{\hat{x}}^{k,N_P}_j\}_{j=1}^{n_p^*}
           $\hfill // Pocket docking
            \STATE  $\{\mathbf{\hat{x}}^{k+1}_i\}_{i=1}^{n_l} \gets $   Ligand Model with the input $(\hat{G}^*, \{\mathbf{h}_i, \mathbf{\hat{x}}^k_i\}_{i=1}^{n_l}, \{\mathbf{h}_j, \mathbf{\hat{x}}^k_j\}_{j=1}^{\hat{n}_{p^*}})$  \hfill // Ligand Player
        \ENDFOR
        \STATE Update model $\mathcal{M}_P$   based on the loss function $\mathcal{J}_P$
    \ENDIF
    \STATE Update ActiongPlayer
\ENDFOR
\end{algorithmic}
\end{algorithm}

\section{LoopPlay: Loop Self-Play}
\begin{figure}
    \centering
    \includegraphics[width=0.8\linewidth]{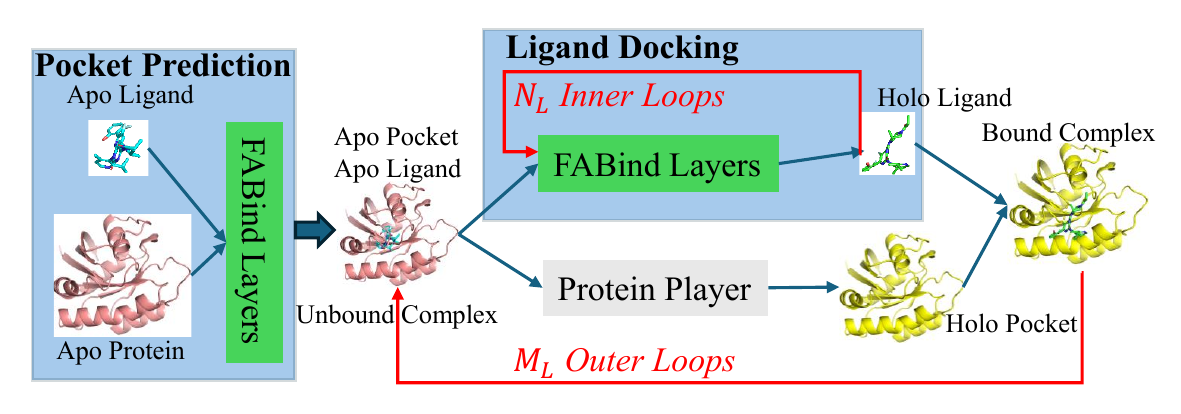}
    \caption{The Ligand Player in LoopPlay.}
    \label{fig:ligand}
\end{figure}
\begin{figure}
    \centering
    \includegraphics[width=0.8\linewidth]{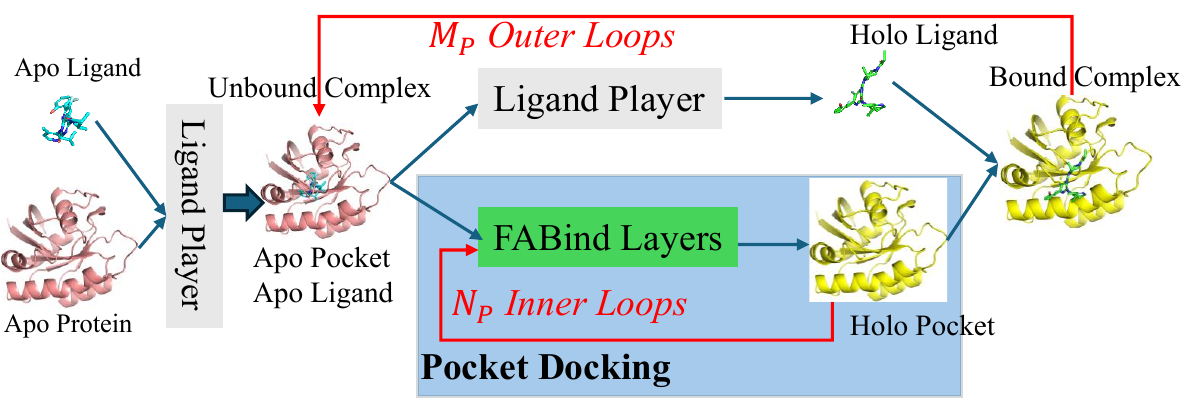}
    \caption{The Protein Player in LoopPlay.}
    \label{fig:protein}
\end{figure}
In this section, we proposed a novel algorithm called \textbf{Loop Self-play (LoopPlay)}, as shown in Algorithm \ref{alg_loopplay}, to solve our \textit{Docking Game}.

Our LoopPlay algorithm is a novel training approach designed to optimize the interaction between two key components: the ligand player and the protein player in our \textit{Docking Game}. 
The algorithm employs an iterative, alternating training strategy to enhance the performance of both players in a cooperative and adversarial manner, inspired by self-play techniques in game theory.

The self-play training process alternates between two distinct phases:
\begin{enumerate}
    \item  \textbf{Ligand Player Training Phase:}
   In this phase, the ligand player is actively trained for one or more epochs, while the protein player is kept fixed (i.e., its parameters are not updated).
  The ligand player focuses on optimizing the pocket prediction and the ligand's pose (position, orientation, and conformation) to achieve better docking accuracy within the protein's binding pocket, as defined by the fixed protein player. This phase allows the ligand player to adapt to the current representation of the protein's binding pocket, effectively learning to propose ligand configurations that maximize docking success. The procedure is shown in Figure \ref{fig:ligand}.
\item \textbf{Protein Player Training Phase:}
   In this phase, the protein player is trained for one or more epochs, while the ligand player is kept fixed.
  The protein player refines its representation of the protein's binding pocket,  effectively learning to propose protein pocket configurations that maximize docking success.
   This procedure is shown in Figure \ref{fig:protein}.
\item \textbf{Iteration:}
    Steps 1 and 2 are repeated iteratively, allowing both players to co-evolve. The ligand player improves its ability to generate optimal ligand poses, while the protein player enhances its ability to generate optimal protein pocket poses. 
\end{enumerate}
 In addition, our LoopPlay framework employs a two-level loop structure for each player to iteratively refine ligand and pocket docking predictions through cooperative interaction between two players: the ligand player and the protein player.
\begin{enumerate}
    \item \textbf{Outer Loop (Cross-Player Learning, i.e., Learning Opponent's Predictions)}
    \begin{itemize}
        \item The two players exchange their latest predictions—the ligand player receives the refined protein pocket pose, and the pocket player receives the refined ligand pose.
        \item Each player then incorporates the opponent’s predictions into their own model, enabling mutual structural adaptation.
        \item Through multiple iterations, this cooperative loop allows both players to progressively improve their predictions by learning from each other’s outputs.
    \end{itemize}
    \item \textbf{Inner Loop (Per-Player Refinement, i.e., Refinement of Individual Predictions)}
    \begin{itemize}
        \item Each player (ligand or protein) dynamically refines their own predictions through multiple iterations.
        \item The predicted ligand pose (from the ligand player) or pocket structure (from the pocket player) is fed back into their own model for further optimization.
        \item This self-refinement process enhances the accuracy of individual predictions before sharing them with the opposing player.
        \item This refinement is the best response to the opponent's pose.
    \end{itemize}
\end{enumerate}
This bidirectional loop feedback mechanism ensures that both ligand and pocket docking structures are optimized in a synergistic manner, leading to more accurate and physically plausible docking results.

As shown in Algorithm \ref{alg_loopplay} and Figures \ref{fig:ligand} and \ref{fig:protein}, the overall procedure is: Starting with an apo ligand generated by RDKit and an apo protein predicted by AlphaFold2, the ligand is initially placed at the protein's center to form the protein--ligand graph. This graph is processed by the pocket prediction module to identify the binding pocket residues. Guided by the predicted pocket, the ligand is repositioned from the protein center to the pocket center, forming the protein pocket--ligand graph (unbound complex). This graph is then input into the ligand and pocket docking modules, where it undergoes two-level loop feedback for refinement. Finally, the predicted holo structures for the ligand and protein pocket (bound complex)   are obtained from their respective docking modules.

 Our LoopPlay leverages a cooperative yet adversarial dynamic between the two players, fostering mutual improvement through loop feedback. It has the following features:
 \begin{itemize}
     \item \textbf{Cooperative Learning:}
  The ligand player and the protein player are interdependent. The ligand player relies on the protein player’s representation of the pocket poses to optimize its docking poses, while the protein player uses the ligand poses to refine its understanding of the pocket docking. By alternating training, each module provides increasingly challenging and informative inputs to the other, driving both to improve.
\item \textbf{Adversarial Dynamics:} The alternating training resembles an adversarial game. Each player learns the best response against the opponent's current strategy. This dynamic process ensures that both players are exposed to progressively harder challenges, preventing overfitting to static or simplistic scenarios.
\item \textbf{Exploration of Solution Space:}
   The loop feedback mechanism encourages exploration of a diverse set of ligand poses and pocket poses. As each player proposes new poses, the other player must adapt to evaluate them, potentially uncovering novel binding modes. 
\end{itemize}

Thus, our LoopPlay is particularly effective for tasks where the interaction between two entities (e.g., ligand and protein) is complex and interdependent. By allowing the ligand player and protein player to iteratively challenge and improve each other, the algorithm achieves a balance between exploration and exploitation, leading to robust and high-performing models for molecular docking. 

\section{Theoretical Properties}

In this section, we discuss the theoretical properties of our LoopPlay algorithm.
\subsection{Convergence}
We demonstrate that LoopPlay converges to a Nash equilibrium, as outlined in Theorem \ref{theorem_convergence}. This convergence is based on the assumption that each optimization step of LoopPlay results in a loss reduction exceeding a certain nonnegative value, ensuring the implementation of a sequence of effective reduction strategies.

Our theoretical findings indicate that the algorithm converges when each self-play optimization step achieves a loss reduction greater than a positive threshold. While this assumption may seem strict, it is consistent with classical iterative optimization frameworks, where consistent improvement is necessary for convergence. In practice, many optimization algorithms—especially those based on gradients—do not always meet this condition at every step but still demonstrate reliable convergence. The key factor is the establishment of a sequence of reduction strategies. Gradient-based algorithms can maintain a downward trend in loss, enabling them to achieve this even when the initial assumption is not strictly fulfilled.
\begin{Theorem}\label{theorem_convergence}
    If each optimization step of LoopPlay achieves a reduction in loss greater than some $ \varepsilon > 0 $ (i.e., $ \mathcal{J}_L(\theta_L^{t+1}, \theta_P^t) < \mathcal{J}_L(\theta_L^t, \theta_P^t) - \epsilon $ or  $ \mathcal{J}_P(\theta_L^t, \theta_P^{t+1}) < \mathcal{J}_P(\theta_L^t, \theta_P^t) - \epsilon $) with compact $\Theta_L$ and $\Theta_P$ and bounded $\mathcal{J}_L$ and $\mathcal{J}_P$, then LoopPlay converges to a Nash equilibrium.
\end{Theorem}
\begin{proof}
We first show that the loss function properties ensure that \textit{the Docking Game} defined by the alternating optimization of players L and P is a potential game.

A potential game \cite{monderer1996potential} is a game in strategic form where there exists a potential function $\mathcal{F} $ such that for each player $ i $, the change in player $i$'s payoff (or loss) when unilaterally changing their strategy matches the change in the potential function $\mathcal{F}$.

To be an exact potential game, there must exist a function $ L(\theta_L, \theta_P) $ for all $\theta_L,\theta'_L\in\Theta_L$, and $\theta_P,\theta'_P\in\Theta_P$, such that: 
\begin{itemize}
    \item For player L: $ \mathcal{J}_L(\theta_L, \theta_P) - \mathcal{J}_L(\theta_L', \theta_P) = \mathcal{F}(\theta_L, \theta_P) - \mathcal{F}(\theta_L', \theta_P) $.
    \item For player P: $ \mathcal{J}_P(\theta_L, \theta_P) - \mathcal{J}_P(\theta_L, \theta'_P) = \mathcal{F}(\theta_L, \theta_P) - \mathcal{F}(\theta_L, \theta'_P) $.
\end{itemize}

Let’s construct a potential function. Consider a candidate potential function:
$$\mathcal{F}(\theta_L, \theta_P) = \alpha_1 \mathcal{L}_{\text{pocket\_pred}}(\theta_L) + \alpha_2 \mathcal{L}_{\text{ligand\_coord}}(\theta_L) + \beta\mathcal{L}_{\text{pocket\_coord}}(\theta_P) + \gamma\mathcal{L}_{\text{dis\_map}}(\theta_L, \theta_P). $$

For player L:
\begin{align*}
    &\mathcal{J}_L(\theta_L, \theta_P) - \mathcal{J}_L(\theta_L', \theta_P)\\
    =& \left[ \alpha_1 \mathcal{L}_{\text{pocket\_pred}}(\theta_L) + \alpha_2 \mathcal{L}_{\text{ligand\_coord}}(\theta_L) + \gamma\mathcal{L}_{\text{dis\_map}}(\theta_L, \theta_P) \right]  \\
    &- \left[ \alpha_1 \mathcal{L}_{\text{pocket\_pred}}(\theta_L') + \alpha_2 \mathcal{L}_{\text{ligand\_coord}}(\theta_L') + \gamma\mathcal{L}_{\text{dis\_map}}(\theta_L', \theta_P) \right]\\
= &\alpha_1 \left[ \mathcal{L}_{\text{pocket\_pred}}(\theta_L) - \mathcal{L}_{\text{pocket\_pred}}(\theta_L') \right] + \alpha_2 \left[ \mathcal{L}_{\text{ligand\_coord}}(\theta_L) - \mathcal{L}_{\text{ligand\_coord}}(\theta_L') \right] \\
    &+ \gamma\left[ \mathcal{L}_{\text{dis\_map}}(\theta_L, \theta_P) -  \mathcal{L}_{\text{dis\_map}}(\theta_L', \theta_P) \right].
\end{align*}

Compute the change in $ \mathcal{F} $:
\begin{align*}
    &\mathcal{F}(\theta_L, \theta_P) - \mathcal{F}(\theta_L', \theta_P)\\ 
    =& \left[ \alpha_1 \mathcal{L}_{\text{pocket\_pred}}(\theta_L) + \alpha_2 \mathcal{L}_{\text{ligand\_coord}}(\theta_L) + \beta\mathcal{L}_{\text{pocket\_coord}}(\theta_P) + \gamma\mathcal{L}_{\text{dis\_map}}(\theta_L, \theta_P) \right] \\
    &- \left[ \alpha_1 \mathcal{L}_{\text{pocket\_pred}}(\theta_L') + \alpha_2 \mathcal{L}_{\text{ligand\_coord}}(\theta_L') + \beta\mathcal{L}_{\text{pocket\_coord}}(\theta_P) + \gamma\mathcal{L}_{\text{dis\_map}}(\theta_L', \theta_P) \right]\\
=& \alpha_1 \left[ \mathcal{L}_{\text{pocket\_pred}}(\theta_L) - \mathcal{L}_{\text{pocket\_pred}}(\theta_L') \right] + \alpha_2 \left[ \mathcal{L}_{\text{ligand\_coord}}(\theta_L) - \mathcal{L}_{\text{ligand\_coord}}(\theta_L') \right] \\
    &+ \gamma\left[ \mathcal{L}_{\text{dis\_map}}(\theta_L, \theta_P) - \mathcal{L}_{\text{dis\_map}}(\theta_L', \theta_P) \right]\\
    =&\mathcal{J}_L(\theta_L, \theta_P) - \mathcal{J}_L(\theta_L', \theta_P)
\end{align*}
This matches exactly, so $ \mathcal{F} $ satisfies the potential function condition for player L.

For player P:
\begin{align*}
    &\mathcal{J}_P(\theta_L, \theta_P) - \mathcal{J}_P(\theta_L, \theta_P')\\
    = &\left[\beta \mathcal{L}_{\text{pocket\_coord}}(\theta_P) + \gamma\mathcal{L}_{\text{dis\_map}}(\theta_L, \theta_P) \right] - \left[\beta \mathcal{L}_{\text{pocket\_coord}}(\theta_P') + \gamma\mathcal{L}_{\text{dis\_map}}(\theta_L, \theta_P') \right]\\
= &\beta\left[ \mathcal{L}_{\text{pocket\_coord}}(\theta_P) - \mathcal{L}_{\text{pocket\_coord}}(\theta_P') \right] + \gamma\left[ \mathcal{L}_{\text{dis\_map}}(\theta_L, \theta_P) - \mathcal{L}_{\text{dis\_map}}(\theta_L, \theta_P') \right].
\end{align*}

Compute the change in $ \mathcal{F} $:
\begin{align*}
    &\mathcal{F}(\theta_L, \theta_P) - \mathcal{F}(\theta_L, \theta_P')\\
    = &\left[ \alpha_1 \mathcal{L}_{\text{pocket\_pred}}(\theta_L) + \alpha_2 \mathcal{L}_{\text{ligand\_coord}}(\theta_L) + \beta\mathcal{L}_{\text{pocket\_coord}}(\theta_P) + \gamma\mathcal{L}_{\text{dis\_map}}(\theta_L, \theta_P) \right]\\
    &- \left[ \alpha_1 \mathcal{L}_{\text{pocket\_pred}}(\theta_L) + \alpha_2 \mathcal{L}_{\text{ligand\_coord}}(\theta_L) + \beta\mathcal{L}_{\text{pocket\_coord}}(\theta_P') + \gamma\mathcal{L}_{\text{dis\_map}}(\theta_L, \theta_P') \right]\\
=& \beta\left[ \mathcal{L}_{\text{pocket\_coord}}(\theta_P) - \mathcal{L}_{\text{pocket\_coord}}(\theta_P') \right] + \gamma\left[ \mathcal{L}_{\text{dis\_map}}(\theta_L, \theta_P) - \mathcal{L}_{\text{dis\_map}}(\theta_L, \theta_P') \right]\\
=&\mathcal{J}_P(\theta_L, \theta_P) - \mathcal{J}_P(\theta_L, \theta_P')
\end{align*}

This also matches exactly, so $ \mathcal{F}$ satisfies the condition for player P.

Thus, the function $ \mathcal{F}(\theta_L, \theta_P) = \alpha_1 \mathcal{L}_{\text{pocket\_pred}}(\theta_L) + \alpha_2 \mathcal{L}_{\text{ligand\_coord}}(\theta_L) + \beta\mathcal{L}_{\text{pocket\_coord}}(\theta_P) + \gamma\mathcal{L}_{\text{dis\_map}}(\theta_L, \theta_P) $ is an exact potential function for the game, because the change in each player’s loss when changing their strategy (parameters) equals the change in $ \mathcal{F} $. 

Second, we show that the self-play structure (alternating optimization) induces an approximate finite improvement property (AFIP) \cite{monderer1996potential}, which is sequence of strategy profiles $ (\theta^0, \theta^1, \ldots) $ where, at each step $ t $, $\theta^t\in \Theta_L \times \Theta_P $, one player $ i $ changes their strategy from $ \theta^{t-1}_i $ to $ \theta^t_i $, and the loss improves by at least $ \epsilon $, i.e., $ \mathcal{J}_i(\theta^t) < \mathcal{J}_i(\theta^{t-1}) - \epsilon $.

$ \mathcal{F}$ is bounded because $ \mathcal{J}_L$ and $ \mathcal{J}_P$ are bounded.  Then the potential game with  $ \mathcal{F}$  has an AFIP.
That is, based on the assumption that each optimization step of LoopPlay achieves a reduction in loss greater than some $ \varepsilon > 0 $ (i.e., $ \mathcal{J}_L(\theta_L^{t+1}, \theta_P^t) < \mathcal{J}_L(\theta_L^t, \theta_P^t) - \epsilon $ or  $ \mathcal{J}_P(\theta_L^t, \theta_P^{t+1}) < \mathcal{J}_P(\theta_L^t, \theta_P^t) - \epsilon $), 
there is  a sequence of strategy profiles $ (\theta^0, \theta^1, \ldots) $ where, at each step $ t $, $\theta^t\in \Theta_L \times \Theta_P $, one player $ i $ changes their strategy from $ \theta^{t-1}_i $ to $ \theta^t_i $, and the loss improves by at least $ \epsilon $, i.e., $ \mathcal{J}_i(\theta^t) < \mathcal{J}_i(\theta^{t-1}) - \epsilon $. This sequence of strategy profiles is an AFIP and finite because $ \mathcal{F}$ is bounded. 

\textit{The Docking Game} is a continuous potential game with compact strategy
 sets, and eventually reaches a pure-strategy equilibrium point \cite{monderer1996potential}. The AFIP generated by our LoopPlay algorithm ensures that LoopPlay converges to a Nash equilibrium \cite{monderer1996potential}.
\end{proof}
\subsection{Generalization}
We use the generalization bound to demonstrate the advantage of our LoopPlay algorithm.

Given a model, with probability at least $1-\delta$, we have the following  generalization bound \cite{mohri2018foundations,bartlett2017spectrally}:
$$\mathbb{E}_{\mathcal{D}}[J] -\frac{1}{n} \sum_{i=1}^n J(\mathbf{x}_i, y_i) \leq  \frac{C_{depth}}{\sqrt{n}} + k \sqrt{\frac{\log(1/\delta)}{2n}}.$$
where
\begin{itemize}
\item $\mathbb{E}_{\mathcal{D}}[J]$: The expected risk, i.e., the average loss of the model over the entire data distribution $\mathcal{D}$.
\item $\frac{1}{n} \sum_{i=1}^n J(\mathbf{x}_i, y_i)$: The empirical risk, i.e., the average loss on the training dataset of $n$ samples $(\mathbf{x}_i, y_i)$.
\item The difference represents the generalization error, or how much worse the model performs on unseen data compared to the training data.
    \item $C_{depth}$: A constant that grows with the depth of the neural network, reflecting the model's capacity (deeper networks represent more complex functions).
\item $n$: Number of training samples.
\item $\delta$ and $k$ are constants. 
\end{itemize}
This generalization bound shows that: Deeper networks (larger $C_{depth}$) or smaller datasets (smaller $n$) increase the risk of overfitting.

Each feedback loop adds computational layers, increasing the network depth, which also augments the data. According to the aforementioned bound, increasing the number of feedback loops enhances the data (effectively increasing the sample size $( n )$) and initially reduces the generalization bound. However, making the model deeper also increases this bound, leading to a trade-off. This implies that LoopPlay can improve generalization by utilizing the appropriate number of loops.


\section{Experimental Evaluation}
\subsection{Settings}
\subsubsection{Dataset.}

We utilize the dataset provided by \cite{zhang2025fast}, which was obtained by processing the PDBBind v2020 dataset \cite{liu2017forging}. This PDBBind v2020 dataset is a widely recognized benchmark in molecular docking research \cite{Pei2023, Lu2024, Lu2022, Corso2023}. PDBBind v2020 includes 19,443 experimentally determined protein-ligand complexes, complete with their 3D structures. 
As the dataset lacks apo protein structures, AlphaFold2 was used to predict the apo conformations of these proteins.    303 complexes recorded after 2019 were designated as the test set, and 734 complexes recorded before 2019 were selected as the validation set, with the remaining complexes used for training. For the training set,   complexes that cannot be processed by RDKit or TorchDrug were excluded, as well as those with protein amino acid chains exceeding 1,500 residues or molecules with more than 150 heavy atoms, resulting in a refined training set of 12,807 complexes. For further details, please refer to \cite{zhang2025fast}.

\subsubsection{Baselines.}
We evaluate our proposed method against a range of baselines:

\begin{enumerate}
    \item Traditional Molecular Docking Software:
    \begin{itemize}
        \item \textbf{Vina} \cite{Trott2010}: AutoDock Vina is a widely used open-source docking program that enhances performance through an improved scoring function based on X-score \cite{Wang2002} and utilizes a Broyden-Fletcher-Goldfarb-Shanno optimization method \cite{Nocedal1999}. It also employs an advanced search strategy.
        \item \textbf{Glide} \cite{Friesner2004}: Glide systematically explores the ligand’s conformational, orientational, and positional spaces relative to a rigid protein receptor. It utilizes hierarchical filters and the ChemScore function \cite{Eldridge1997} to refine its docking process.
        \item \textbf{Gnina} \cite{McNutt2021}:  Gnina integrates convolutional neural networks into its scoring function and employs Monte Carlo sampling techniques to thoroughly explore the conformational space of ligands
    \end{itemize}
    \item  Deep Learning-Based Methods with Rigid Protein Assumption:
    \begin{itemize}
        \item \textbf{TankBind} \cite{Lu2022}:   TankBind utilizes P2Rank \cite{Krivak2018}   to identify binding pockets by dividing the protein into functional blocks. It then applies a trigonometry-aware graph neural network to model protein-ligand interactions, predict distance matrices, and optimize ligand structures.
        \item \textbf{FABind} \cite{Pei2023}:   FABind features an end-to-end framework that combines binding pocket prediction with docking, streamlining the process by eliminating the need for external pocket detection tools.
        \item \textbf{FABind+} \cite{gao2025fabind+}: An improved version of FABind, FABind+ incorporates dynamic adjustments of pocket radius and uses permutation-invariant loss to enhance the accuracy of predicted ligand structures.
        \item \textbf{DiffDock} \cite{Corso2023}:   DiffDock employs diffusion models \cite{Yang2023}   to refine ligand conformations, treating docking as a generative modeling task on a non-Euclidean manifold to reduce degrees of freedom. It also includes a confidence model to assess poses generated from multiple samplings.
        \item \textbf{DiffDock-L} \cite{Corso2024}: An advanced iteration of DiffDock, DiffDock-L improves performance by scaling data and model size while incorporating synthetic data to enhance generalization.
    \end{itemize}
 \item Recent Deep Learning-Based Flexible Docking Methods: \textbf{DynamicBind} \cite{Lu2024} and \textbf{FABFlex} \cite{zhang2025fast} are introduced in the related work section. 
\end{enumerate}

 \begin{table}[]
      \caption{Parameters for training LoopPlay}
     \label{tab:implementation}
     \centering
     \begin{tabular}{|c|c|c|c|c|c|c|c|c|c|c|}\hline
         Learning rate &Epoch&Batch size&Dropout&Optimizer&Scheduler &$\alpha^{cls}_1$&$\alpha^{center}_1$&$\alpha_2$&$\beta$&$\gamma$  \\\hline
        $5e-5$  &200 &4&0.1&Adam&LinerLR&1&0.05&50&15&1\\ \hline
     \end{tabular}
 \end{table}

\subsubsection{Model Configuration.} In our experiments, we set the number of loops in LoopPlay as follows: \(M_L = M_P = 2\) and \(N_L = N_P = 6\). The FABind model configuration consists of 1 layer for the pocket prediction module of the ligand player, 5 layers for the ligand docking module of the ligand player, and 5 layers for the protein pocket docking module of the protein player. The corresponding hidden sizes for these modules are \{128, 512, 512\}. Additional parameter settings can be found in Table \ref{tab:implementation}. We initialized the modules of both the ligand player and the protein player using the pretrained models provided by \cite{zhang2025fast}. The experiments were conducted using the PyTorch framework on eight NVIDIA A800 80GB GPUs.     
\begin{table}[h]
\centering
\caption{Ligand performance comparison of blind flexible docking: The top results are shown in bold, the second-best results are underlined, and results that are the best only among learning-based methods are italicized. The average runtime for each method is reported in seconds.}
\label{tab:ligand_performance}
\scalebox{0.9}{
\begin{tabular}{lcccccc|ccccccc} 
\toprule
 & \multicolumn{5}{c}{On All Cases} & & \multicolumn{6}{c}{On Unseen Protein Receptors} & \\
\cmidrule(lr){2-6} \cmidrule(lr){8-13}
  & \multicolumn{4}{c}{Percentiles $\downarrow$} & \multicolumn{2}{c}{\% Below $\uparrow$}  & \multicolumn{4}{c}{Percentiles $\downarrow$} & \multicolumn{2}{c}{\% Below $\uparrow$} &  Average\\
  \cmidrule(lr){2-5} \cmidrule(lr){6-7} \cmidrule(lr){8-11} \cmidrule(lr){12-13}
Algorithm & {25\%} & {50\%} & {75\%} & {Mean} & {$<2\ \si{\angstrom}$} & {$<5\ \si{\angstrom}$}  & {25\%} & {50\%} & {75\%} & {Mean} & {$<2\ \si{\angstrom}$} & {$<5\ \si{\angstrom}$} &  Runtime(s)\\
\midrule
\multicolumn{14}{c}{\textit{Traditional Docking Software}} \\
Vina &4.79& 7.14& 9.21& 7.14& 6.67 &27.33& 5.27& 7.06& 8.84& 7.15& 6.25& 23.21& 205\\
 Glide &2.84 &5.77 &8.04 &5.81 &14.66& 40.60& 2.38& 5.01 &\textbf{7.17} &\textbf{5.21} &21.36& 49.51 &1405\\
 Gnina &2.58& 5.17 &8.42& 5.76& 19.32 &48.47& 2.03 &4.96& 7.35 &\underline{5.33}& 24.55& 50.91 &146\\
\multicolumn{14}{c}{\textit{Deep Learning-Based   Docking Methods}} \\
TankBind& 2.82 &4.53& 7.79 &7.79& 8.91 &54.46& 2.88& 4.45 &7.53& 7.60 &4.39& 58.77 &0.87\\
 FABind& 2.19& 3.73 &8.39& 6.63& 22.11& 60.73& 2.73& 4.83& 9.35& 7.15& 8.77& 50.88& 0.12\\
 FABind+ &1.58 &\underline{2.79}& 6.69 &5.63& 35.64& 66.01& 1.93& \textbf{3.13} &8.59& 6.76& 27.19& 57.89 &0.16\\
 DiffDock& 1.82& 3.92& 6.83& 6.07& 29.04& 60.73& 1.97 &4.82 &8.03 &7.41& 26.32 &51.75 &82.83\\
 DiffDock-L& 1.55& 3.22& 6.86& 5.99& 36.75& 62.58& \underline{1.86}& \underline{3.16} &9.09& 7.14 &\underline{29.82}& \underline{61.40} &58.72\\
  DynamicBind &1.57 &3.16& 7.14& 6.19& 33.00& 64.69 &2.23 &4.02& 10.23 &8.27 &20.18& 54.39& 102.12\\
  FABFlex& \textbf{1.40} &2.96& \underline{6.16}& \underline{5.44}& \underline{40.59}& \underline{68.32}&\textbf{1.81} &3.51& 8.03 &7.17 &\textbf{32.46}& 59.65& 0.17\\
  \hline LoopPlay&\underline{1.49}&\textbf{2.78}&\textbf{5.70}&\textbf{4.90}&\textbf{41.91}&\textbf{71.95}&\textbf{1.81}&3.33&\underline{\it 7.32}&\textit{6.59}&\textbf{32.46}&\textbf{62.28}&0.32\\
\bottomrule
\end{tabular}
}
 \end{table}

\begin{table}[h]
\centering
\caption{Ablation Study on ligand performance comparison of blind flexible docking. LoopPlay (M,N) represents $M_L=M_P=M$ and $N_L=N_P=N$.}
\label{tab:ablation_performance}
\scalebox{0.9}{
\begin{tabular}{lcccccc|ccccccc} 
\toprule
 & \multicolumn{5}{c}{On All Cases} & & \multicolumn{6}{c}{On Unseen Protein Receptors} & \\
\cmidrule(lr){2-6} \cmidrule(lr){8-13}
  & \multicolumn{4}{c}{Percentiles $\downarrow$} & \multicolumn{2}{c}{\% Below $\uparrow$}  & \multicolumn{4}{c}{Percentiles $\downarrow$} & \multicolumn{2}{c}{\% Below $\uparrow$} &  Average\\
  \cmidrule(lr){2-5} \cmidrule(lr){6-7} \cmidrule(lr){8-11} \cmidrule(lr){12-13}
Algorithm & {25\%} & {50\%} & {75\%} & {Mean} & {$<2\ \si{\angstrom}$} & {$<5\ \si{\angstrom}$}  & {25\%} & {50\%} & {75\%} & {Mean} & {$<2\ \si{\angstrom}$} & {$<5\ \si{\angstrom}$} &  Runtime(s)\\
\midrule
  FABFlex& 1.40 &2.96& {6.16}& {5.44}& 40.59& {68.32}&\underline{1.81} &3.51& 8.03 &7.17 &\textbf{32.46}& 59.65& 0.17\\
  \hline LoopPlay&{1.49}&{2.78}&{5.70}&{4.90}&\textbf{41.91}&\textbf{71.95}&\underline{1.81}&{3.33}&{7.32}&{6.59}&\textbf{32.46}&{62.28}&0.32\\ \hline 
  LoopPlay (1,1)&2.0&3.58&6.71&5.59&24.75&66.01&2.43&4.16&8.44&7.15&14.91&57.02         &0.04\\
  LoopPlay (2,1)&1.59&2.97&5.65&5.04&34.98&71.61&2.20&3.44&7.70&6.77&21.05&63.16  &0.07\\
  LoopPlay (2,2)&1.42&2.70&5.56&5.01&38.28&69.31&1.92&3.64&7.28&6.68&26.32&60.53 &0.12\\
  LoopPlay (2,3)&1.47&\underline{2.62}&5.93&5.05&40.26&69.64&\textbf{1.78}&3.45&7.46&6.73&\underline{30.70}&63.16 &0.17\\
  LoopPlay (2,4)&\textbf{1.38}&2.82&\underline{5.55}&4.89&39.60&69.97&1.93&3.57&7.13&6.65&26.32&59.65 &0.23\\
  LoopPlay (2,5) &   1.42& \textbf{2.61}  & 5.76& 4.92& 41.58 & 70.63    & 1.88 & 3.28& 6.91 & 6.48 & 27.19 & \underline{64.04} &0.28\\
  LoopPlay (2,6)&{1.49}&{2.78}&{5.70}&{4.90}&\textbf{41.91}&\textbf{71.95}&\underline{1.81}&{3.33}&{7.32}&{6.59}&\textbf{32.46}&{62.28}&0.32\\
  LoopPlay (2,7)&   \underline{1.39} & 2.83 &6.16 & 5.18 &  38.28  & 68.32 & 1.91 &3.40 & 8.22 & 6.90 & 25.44& 58.77  &0.39\\
  LoopPlay (2,8)&1.58&2.92&\textbf{5.44}&\underline{4.86}&36.30&\textbf{71.95}&2.15&3.29&\underline{6.72}&\textbf{6.22}&22.81&\textbf{66.67} &0.45\\
  LoopPlay (2,9)&   1.52 & 2.74 & 5.97 & 5.00 &\underline{41.58} & \underline{71.62} &   1.83 & \underline{3.25} & 6.75 & 6.57& \underline{30.70} & 63.16  &0.50\\
  LoopPlay (2,10)&1.42&2.73&5.87&\textbf{4.82}&39.93&69.97&1.91&\textbf{3.23}&\textbf{6.65}&\underline{6.30}&27.19&62.28 &0.55\\
\bottomrule
\end{tabular}
}
 \end{table}

\subsubsection{Evaluation Metric.} The evaluation metric known as Ligand RMSD (Root Mean Square Deviation) quantifies the difference between the predicted and actual Cartesian coordinates of ligand atoms. This metric reflects the model's competence in accurately identifying the ligand's conformation at the atomic level. 
The RMSD between the predicted and true atomic Cartesian coordinates of the ligand is calculated as follows:

\[
\text{RMSD} = \sqrt{\frac{1}{n} \sum_{i=1}^{n} \left( (x^1_i - \mathbf{\hat{x}}^1_i)^2 + (x^2_i - \mathbf{\hat{x}}^2_i)^2 + (x^3_i - \mathbf{\hat{x}}^3_i)^2 \right
)}
\]

where:
\begin
{itemize}
    \item $n$
 is the number of atoms,
    \item $(x^1_i, x^2_i, x^3_i)$ are the true coordinates of the $i$-th atom,
    \item $(\mathbf{\hat{x}}^1_i, \mathbf{\hat{x}}^2_i, \mathbf{\hat{x}}^3_i)$ are the predicted coordinates of the $i$-th atom.
\end
{itemize}
\subsection{Results}
\subsubsection{Performance in Blind Flexible Docking}
A comparison of ligand performance across different docking methods is presented in Table \ref{tab:ligand_performance}. The data in the left section of the table demonstrates that LoopPlay consistently outperforms both traditional docking software and modern deep learning-based methods across nearly all metrics. It ranks either as the best or second-best overall, excelling in almost every evaluated parameter. Notably, LoopPlay achieves approximately a 10\% improvement in average RMSD compared to the previous state-of-the-art method, FABFlex.

A ligand structure prediction is considered successful if its RMSD is within 2$\si{\angstrom}$ of the true holo ligand structure \cite{Lu2024, zhang2025fast}. Therefore, achieving an RMSD of less than 2$\si{\angstrom}$ is a crucial metric for evaluating molecular docking methods. LoopPlay attains a success rate of 41.91\% for this specific metric, significantly surpassing baseline methods and highlighting its superior ability to predict accurate ligand binding structures.

\subsubsection{Performance in Blind Flexible Docking with Unseen Proteins} To conduct a more rigorous evaluation, we implemented a filtering step to exclude samples with protein UniProt IDs that were not present in the training data. This stricter assessment utilized 114 protein-ligand complexes featuring protein receptors that were unseen during training, allowing us to test the generalization ability of each method. The results of this evaluation can be found in the right section of Table \ref{tab:ligand_performance}. 

LoopPlay outperforms previous methods across most metrics. It consistently ranks as either the best or second-best in a majority of these metrics, achieving the top rank in half of them. Notably, LoopPlay surpasses modern deep learning-based methods in nearly all metrics. In terms of average RMSD, LoopPlay shows an improvement of approximately 8\% compared to the previous state-of-the-art deep learning method, FABFlex. Furthermore, LoopPlay records the highest success rate of 32.46\% in the ligand RMSD < 2$\si{\angstrom}$ metric, highlighting its superior ability to generalize to new proteins.

\subsubsection{Inference Efficiency} 
High efficiency is crucial for the widespread adoption of methods in practical applications. The rightmost column of Table \ref{tab:ligand_performance} compares the average inference time for each protein-ligand pair. Traditional docking tools, such as Vina, Glide, and Gnina, have significantly longer inference times. Among regression-based methods, TankBind, FABind, FABind+, FABFlex, and our LoopPlay are notably faster than sampling-based methods like DiffDock and DynamicBind. Remarkably, LoopPlay achieves an inference speed of just 0.32 seconds, making it approximately 319 times faster than DynamicBind, a recent flexible docking sampling-based method that averages 102.12 seconds. As regression-based methods, the speed of LoopPlay is comparable to that of TankBind, FABind, FABind+, and FABFlex.
\begin{table}[tp]
\centering
\caption{Ablation Study on ligand performance (centroid distance) comparison of blind flexible docking. LoopPlay (M,N) represents $M_L=M_P=M$ and $N_L=N_P=N$}
\label{tab:ablation_centroid_dis}
\scalebox{0.88}{
\begin{tabular}{lcccccc|cccccccc} 
\toprule
 & \multicolumn{5}{c}{On All Cases} & & \multicolumn{6}{c}{On Unseen Protein Receptors} & &\\
\cmidrule(lr){2-6} \cmidrule(lr){8-13}
  & \multicolumn{4}{c}{Percentiles $\downarrow$} & \multicolumn{2}{c}{\% Below $\uparrow$}  & \multicolumn{4}{c}{Percentiles $\downarrow$} & \multicolumn{2}{c}{\% Below $\uparrow$} &  Pocket&Pocket\\
  \cmidrule(lr){2-5} \cmidrule(lr){6-7} \cmidrule(lr){8-11} \cmidrule(lr){12-13}
Algorithm & {25\%} & {50\%} & {75\%} & {Mean} & {$<2\ \si{\angstrom}$} & {$<5\ \si{\angstrom}$}  & {25\%} & {50\%} & {75\%} & {Mean} & {$<2\ \si{\angstrom}$} & {$<5\ \si{\angstrom}$} &  Accuracy&RMSD\\
\midrule
  FABFlex&  
  0.56&1.11&2.78&3.67&67.33&83.50&0.72&1.45&5.01&5.27&60.53&74.56&87.08&1.10\\
 LoopPlay&0.58&1.03&2.75&3.11&65.68&87.46&\textbf{0.69}&\underline{1.38}&4.20&4.70&56.14&78.95&86.46&1.19\\\hline
  LoopPlay (1,1)&0.86&1.66&3.02&3.63&59.41&84.49&1.22&1.92&4.21&5.11&50.88&76.32  &82.27&1.02\\
  LoopPlay (2,1)&0.59&1.29&\textbf{2.37}&3.25&68.98&86.80&1.03&1.84&3.94&4.97&56.14&78.95&82.97&1.05\\
  LoopPlay (2,2)&0.56&1.12&2.54&3.22&68.65&86.80&0.87&1.81&3.53&4.80&57.89&78.95&86.71&1.05\\
  LoopPlay (2,3)&0.57&1.06&2.51&3.21&68.98&87.13&\underline{0.70}&1.48&3.90&4.98&59.65&77.19&86.47&1.05\\
  LoopPlay (2,4)&\textbf{0.50}&1.07&2.52&3.12&67.66&\underline{87.79}&0.78&2.90&4.16&4.95&48.25&78.95&86.25&1.09\\
  LoopPlay (2,5)& \underline{0.52} &1.13 & 2.57& 3.11 & \textbf{69.64} & 87.13&  0.72 &1.66&\underline{3.15} & 4.58 & 59.65  & 78.95   &86.45&1.09\\
  LoopPlay (2,6)&0.58&1.03&2.75&3.11&65.68&87.46&\textbf{0.69}&\underline{1.38}&4.20&4.70&56.14&78.95&86.46&1.19\\
  LoopPlay (2,7)&    0.55 & \underline{1.01} & \underline{2.48}& 3.31& 68.32 & 85.15 & 0.77 &1.56& 3.71 &4.80 &59.65 & 78.95    &86.16&1.10\\
  LoopPlay (2,8)&0.53&1.08&2.50&\underline{3.06}& 68.32 &\textbf{88.45} &0.80&1.47&3.32&\underline{4.36}&57.02&\textbf{81.58}&87.46&1.11\\
  LoopPlay (2,9)&  0.53 & 1.03 &\textbf{2.37}  & 3.24 & \textbf{69.64}  &87.13 &  0.84  & \textbf{1.31} & 3.17  & 4.73 & \textbf{63.16} & \underline{80.70} &86.53 &1.10\\
  LoopPlay (2,10)&\textbf{0.50}&\textbf{0.95}& 2.55 & \textbf{2.99} & \underline{69.31} & 87.13  & \underline{0.70}  & 1.52 &\textbf{3.11} & \textbf{4.25}& \underline{61.40} & \underline{80.70}  &86.97&1.12\\
\bottomrule
\end{tabular}
}
 \end{table}

\begin{figure}[htp]
    \centering
        \begin{subfigure}[b]{0.33\textwidth}
        \centering
        \includegraphics[width=\textwidth]{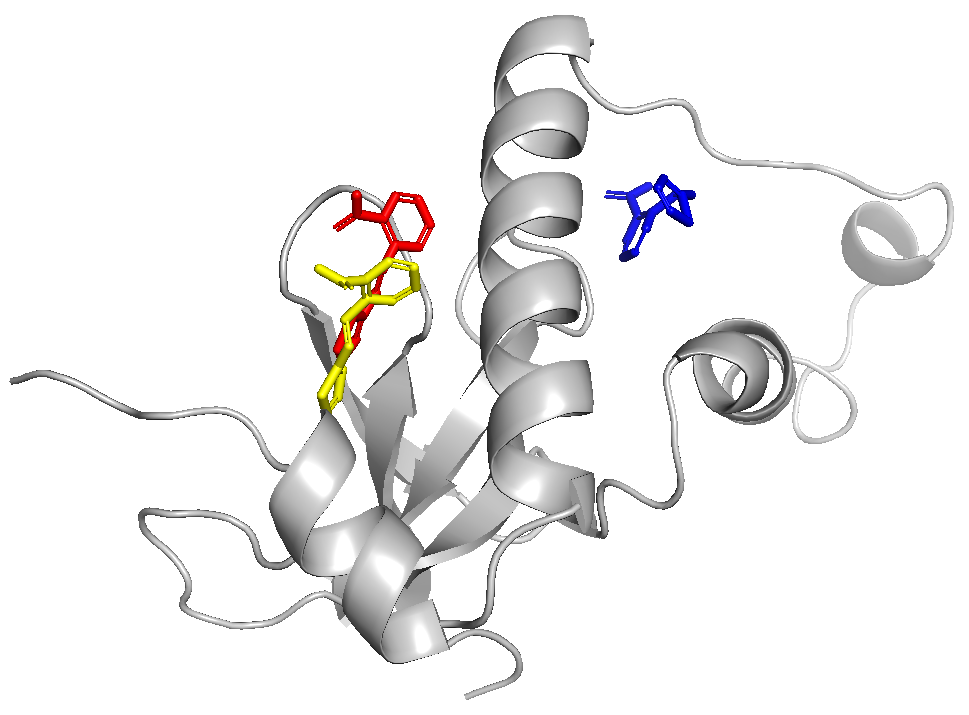}
        \caption{PDB: 6G24 (RMSD:  \textcolor{yellow}{4.34} vs. \textcolor{blue}{16.84})}
        \label{fig:sub1}
    \end{subfigure}       
        \hfill
    \begin{subfigure}[b]{0.33\textwidth}
        \centering
        \includegraphics[width=\textwidth]{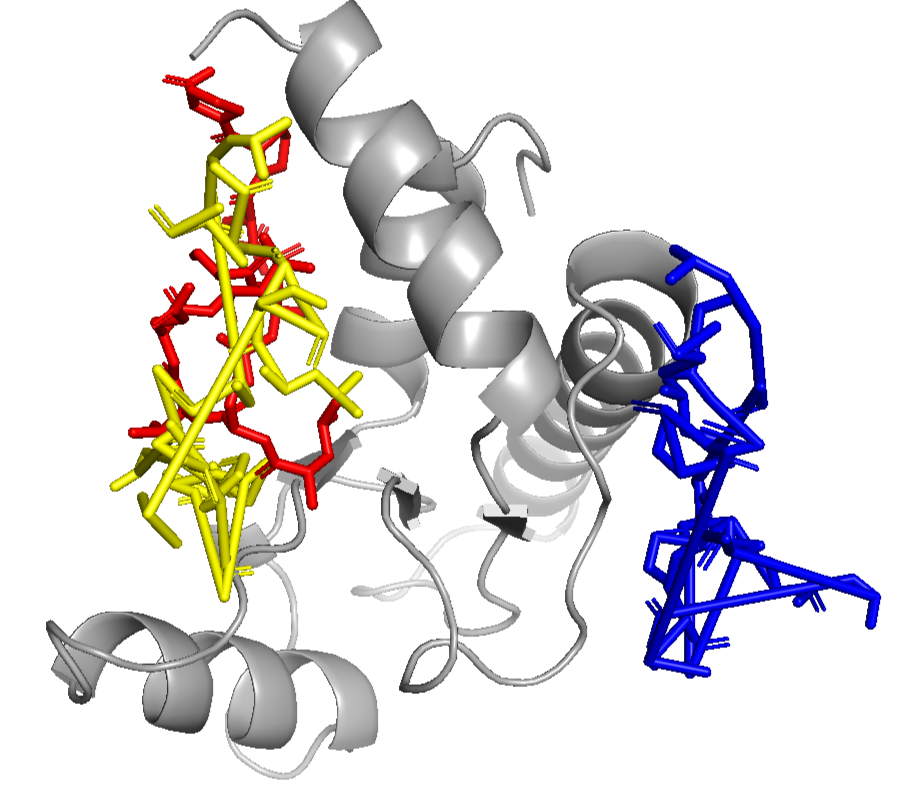}
        \caption{PDB: 6RR0 (RMSD: \textcolor{yellow}{12.42} vs. \textcolor{blue}{26.20})}
        \label{fig:sub1}
    \end{subfigure}
     \hfill
    \begin{subfigure}[b]{0.33\textwidth}
        \centering
        \includegraphics[width=\textwidth]{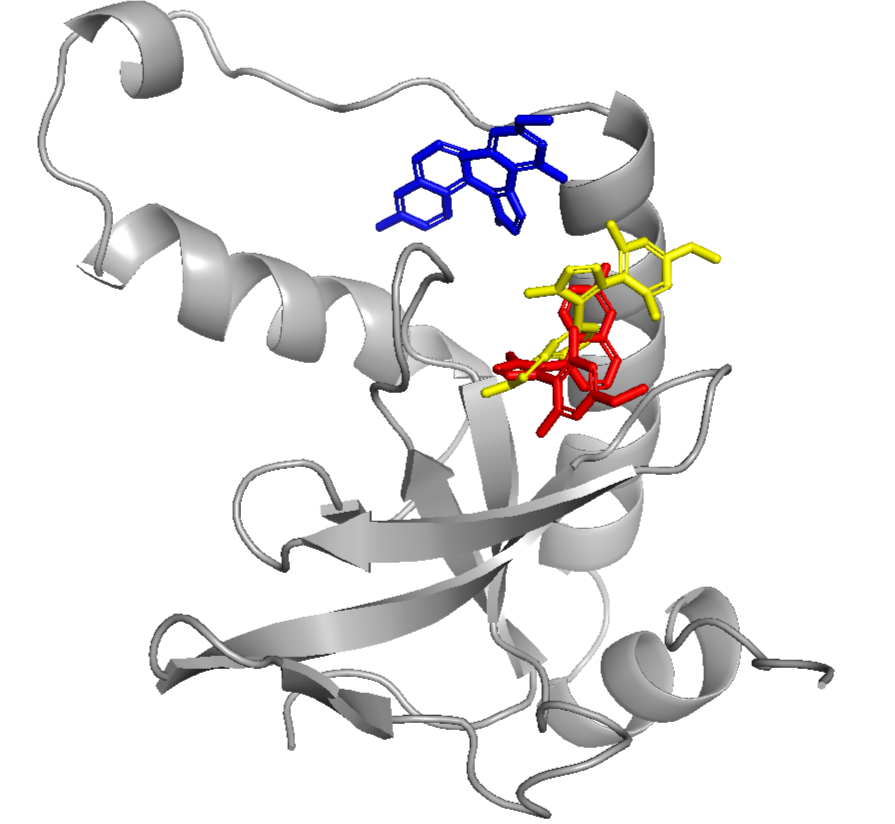}
        \caption{PDB: 6G2O (RMSD:  \textcolor{yellow}{6.52} vs. \textcolor{blue}{11.94})}
        \label{fig:sub1}
    \end{subfigure}
     \vspace{0.5cm} 
    \begin{subfigure}[b]{0.33\textwidth}
        \centering
        \includegraphics[width=\textwidth]{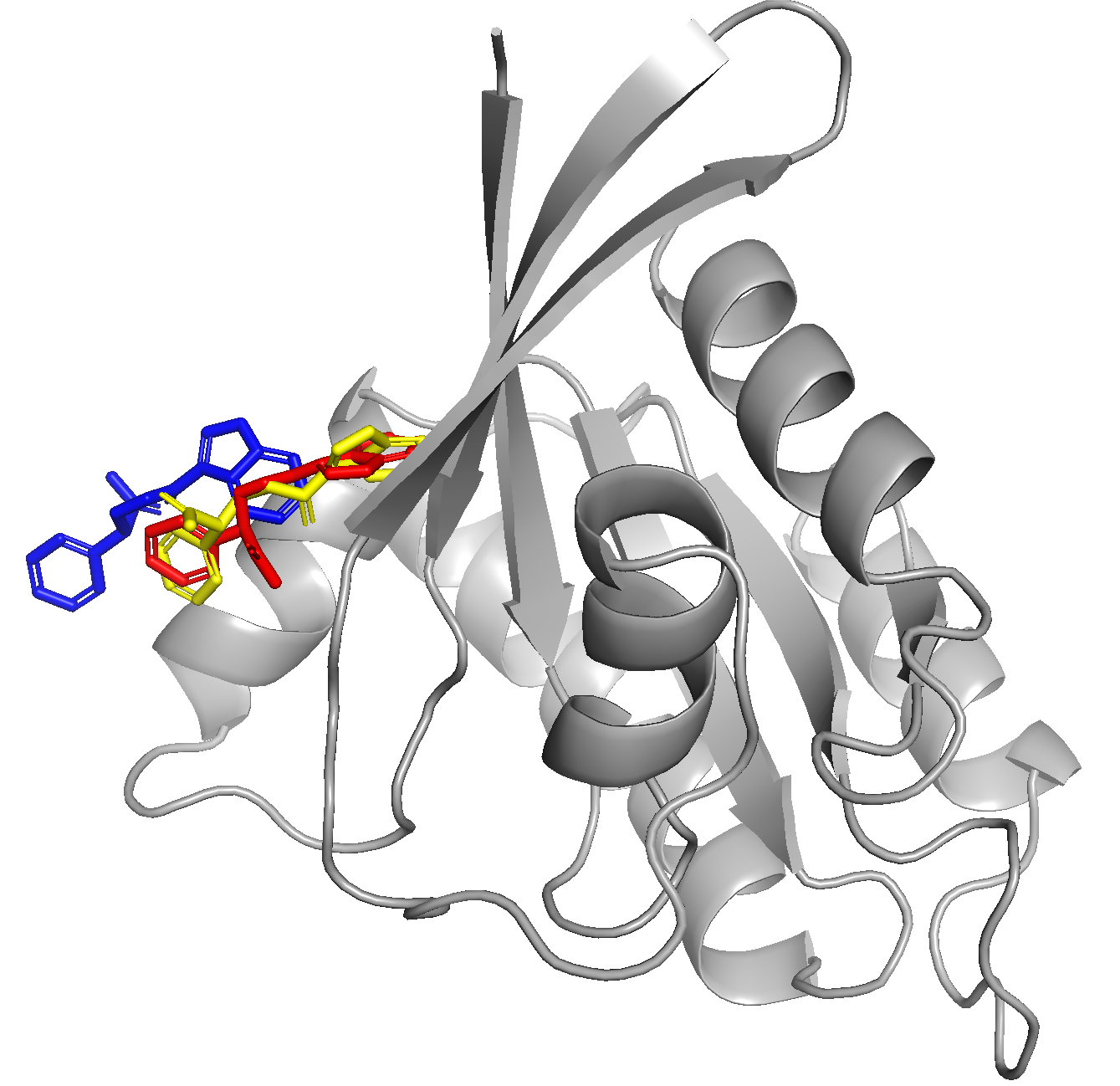}
        \caption{PDB: 6V5L (RMSD: \textcolor{yellow}{1.60} vs.  \textcolor{blue}{5.61})}
        \label{fig:sub2}
    \end{subfigure}
    \hfill
    \begin{subfigure}[b]{0.33\textwidth}
        \centering
        \includegraphics[width=\textwidth]{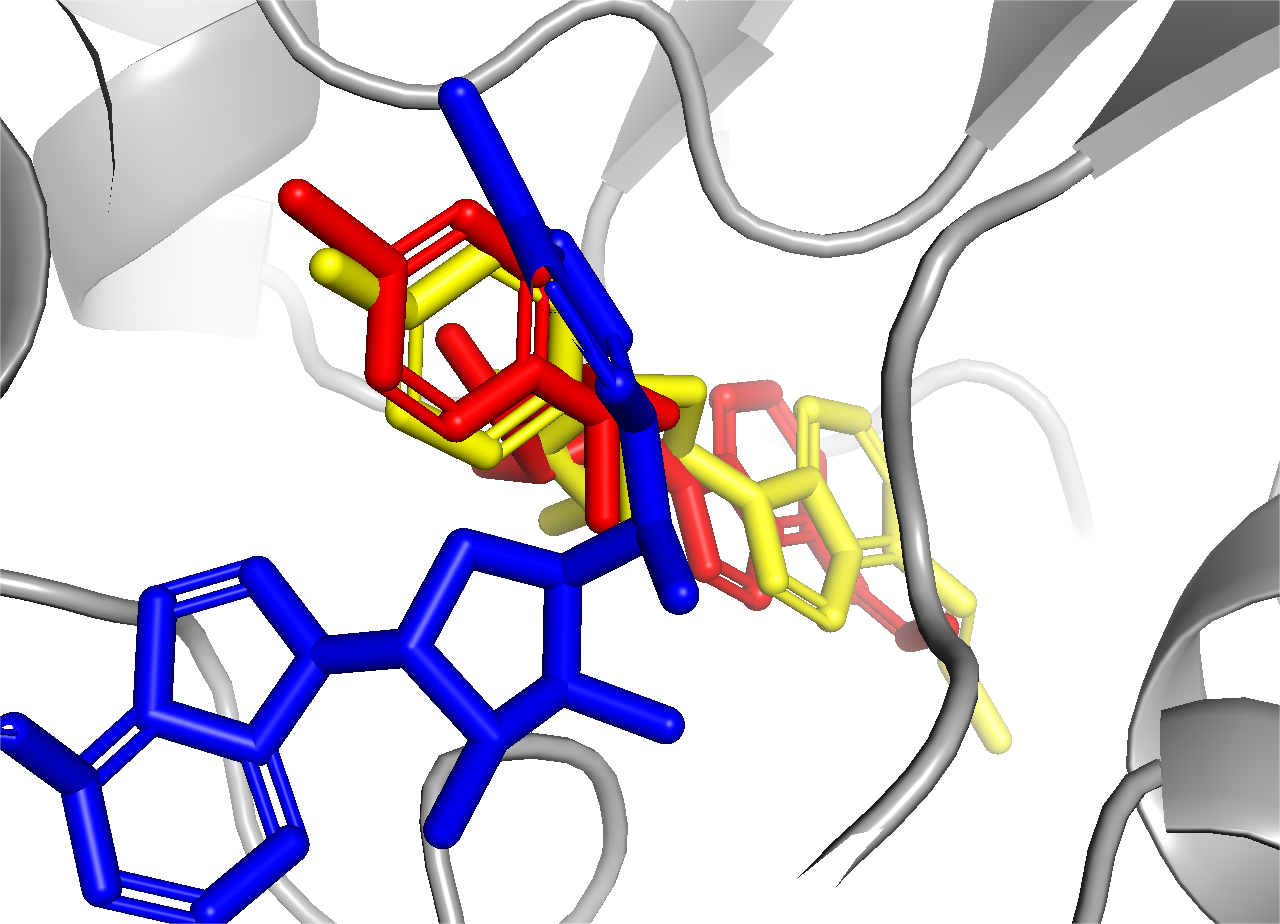}
        \caption{PDB: 6K1S (RMSD:  \textcolor{yellow}{1.46} vs.   \textcolor{blue}{9.97})}
        \label{fig:sub3}
    \end{subfigure}
    \hfill
    \begin{subfigure}[b]{0.33\textwidth}
        \centering
        \includegraphics[width=\textwidth]{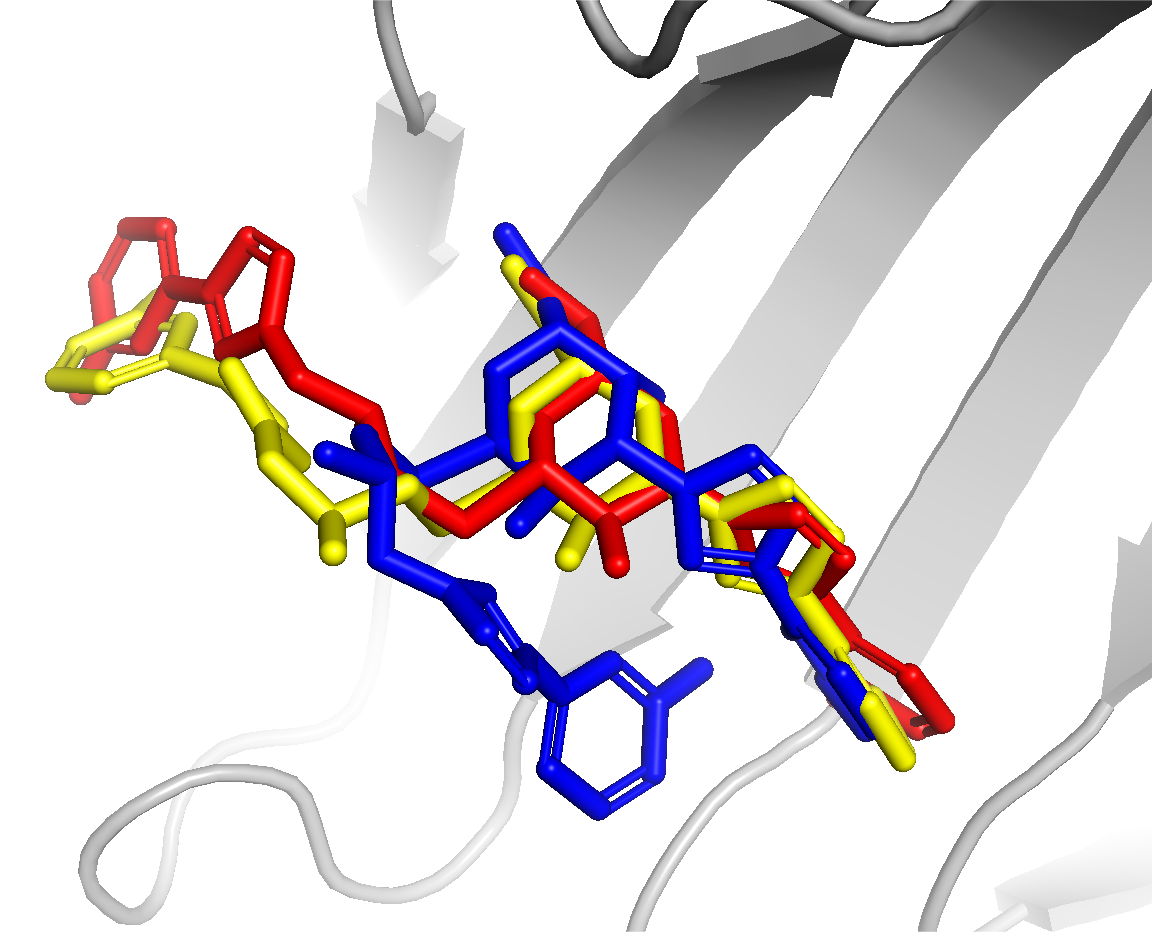}
        \caption{PDB: 6QGE (RMSD: \textcolor{yellow}{1.68} vs.  \textcolor{blue}{6.77})}
        \label{fig:sub4}
    \end{subfigure}
    
    \caption{\textcolor{red}{Red ligand}: ground truth; \textcolor{yellow}{Yellow ligand}: by LoopPlay; \textcolor{blue}{Blue ligand}: by FABFlex.}
    \label{fig:case_study}
\end{figure}
\subsection{Further Analysis}
\subsubsection{Ablation Study.} 

In this study, we focus on the original version of Self-Play, where there are no inner or outer loops; specifically, we have \(N_L = M_L = N_P = M_P = 1\). We also examine the impact of the number of loops on the results. Alongside the ligand root-mean-square deviation (RMSD), we introduce a new metric called Centroid Distance. This metric measures the Euclidean distance between the predicted and actual averaged ligand coordinates, which indicates the model's ability to accurately identify the correct binding site. The formula for calculating the centroid distance between the predicted and true averaged ligand coordinates is provided below:
\[
\text{Centroid Distance} = \sqrt{(x^1_c - \mathbf{\hat{x}}^1_c)^2 + (x^2_c - \mathbf{\hat{x}}^2_c)^2 + (x^3_c - \mathbf{\hat{x}}^3_c)^2}
\]
where:
\begin{itemize}
    \item $(x^1_c, x^2_c, x^3_c)$ are the true centroid coordinates of the ligand, calculated as the average of the atomic coordinates,
    \item $(\mathbf{\hat{x}}^1_c, \mathbf{\hat{x}}^2_c, \mathbf{\hat{x}}^3_c)$ are the predicted centroid coordinates of the ligand, calculated as the average of the predicted atomic coordinates.
\end{itemize}
Table \ref{tab:ablation_performance} presents the results for ligand RMSD, while Table \ref{tab:ablation_centroid_dis} shows the results for centroid distance in the ablation study. It is important to note that LoopPlay (1,1) represents the original version of Self-Play, which performs the worst among the options, in terms of the RMSD and the centroid distance.  It not only underperforms compared to other versions of LoopPlay, but also significantly underperforms the baseline method, FABFlex. As the number of loops increases, LoopPlay shows improved performance to some extent. However, this does not imply that the performance of LoopPlay consistently rises with the number of loops, especially in terms of high-accuracy metrics, which aligns with our theoretical analysis.

 \subsubsection{Efficiency and Pocket Performance}
In the last column of Table \ref{tab:ablation_performance}, it can be seen that the inference time for LoopPlay increases with the number of loops. However, all inference times remain quite fast, each under 1 second. This speed is significantly faster than the sampling methods shown in Table \ref{tab:ligand_performance}. Notably, the speed of LoopPlay (2,3) is comparable to that of FABFlex, yet LoopPlay (2,3) significantly outperforms FABFlex across most metrics.

Table \ref{tab:ablation_centroid_dis} demonstrates that the accuracy of protein pocket prediction is high across nearly all tested algorithms. When both the number of inner and outer loops exceeds 2, the pocket prediction accuracies among different versions of LoopPlay are very similar and also comparable to that of FABFlex. Additionally, the Pocket RMSDs of the various tested algorithms, presented in Table \ref{tab:ablation_centroid_dis}, are also small and closely aligned. These results indicate that our LoopPlay enhances the performance of ligand docking predictions without compromising the efficacy of protein pocket docking predictions.

 \subsection{Case Study}
Figure \ref{fig:case_study} illustrates four cases: PDB 6G24, PDB 6V5L, PDB 6K1S, and PDB 6QGE. These cases demonstrate LoopPlay's effectiveness in blind flexible docking. There are two key observations:

\begin{itemize}
    \item LoopPlay accurately identifies binding pocket sites. For PDB 6G24, 6RR0, and 6G2O, FABFlex incorrectly predicts the binding sites, while LoopPlay correctly locates them, achieving a notably lower ligand RMSD. This result highlights LoopPlay's superior ability to pinpoint binding pockets.
    
    \item LoopPlay excels in ligand structure prediction. For PDB 6V5L, PDB 6K1S, and PDB 6QGE, although all methods correctly identify the binding pocket, LoopPlay's predicted ligand structures are much closer to the ground truth, with ligand RMSDs of 1.60 $\si{\angstrom}$, 1.46$\si{\angstrom}$, and 1.68$\si{\angstrom}$, respectively, all of which are under 2$\si{\angstrom}$. This low RMSD indicates that LoopPlay’s predicted ligands are nearly identical to the true holo ligand, showcasing its exceptional proficiency in ligand structure prediction. In contrast, the predicted ligand structures from FABFlex in these cases deviate significantly from the ground truth, resulting in higher RMSDs, as illustrated in Figure \ref{fig:case_study}.
\end{itemize}

\section{Conclusion}
To enhance the performance of protein-ligand docking, we propose a novel game-theoretic framework that addresses this issue by modeling protein-ligand interactions as a two-player \textit{Docking Game}. In this framework, the ligand docking module acts as the ligand player, while the protein pocket docking module serves as the protein player. 
We introduce a novel algorithm called LoopPlay, designed to train these players through a two-level loop structure. The outer loop facilitates mutual adaptation by allowing the players to exchange predicted poses, while the inner loop refines each player's predictions based on its own outputs. 
We provide a theoretical proof of LoopPlay’s convergence, ensuring stable optimization. Extensive experiments conducted on public benchmark datasets demonstrate approximately a 10\% improvement in accurately predicting binding modes compared to previous state-of-the-art methods. These results highlight LoopPlay’s potential to significantly enhance the precision and reliability of molecular docking, thereby advancing its application in drug discovery.

 \bibliography{aidd}

\begin{thebibliography}{10}

\bibitem{agarwal2016overview}
Shweta Agarwal and RJJC Mehrotra.
\newblock An overview of molecular docking.
\newblock {\em JSM Chem}, 4(2):1024--1028, 2016.

\bibitem{bartlett2017spectrally}
Peter~L Bartlett, Dylan~J Foster, and Matus Telgarsky.
\newblock Spectrally-normalized margin bounds for neural networks.
\newblock In {\em Proceedings of the 31st International Conference on Neural Information Processing Systems}, pages 6241--6250, 2017.

\bibitem{berman2000protein}
Helen~M Berman, John Westbrook, Zukang Feng, Gary Gilliland, Talapady~N Bhat, Helge Weissig, Ilya~N Shindyalov, and Philip~E Bourne.
\newblock The protein data bank.
\newblock {\em Nucleic acids research}, 28(1):235--242, 2000.

\bibitem{brown:fp1951}
George~W. Brown.
\newblock Iterative solution of games by fictitious play.
\newblock In T.~C. Koopmans, editor, {\em Activity Analysis of Production and Allocation}. Wiley, New York, 1951.

\bibitem{uniprot2019uniprot}
UniProt Consortium.
\newblock {UniProt}: {A} worldwide hub of protein knowledge.
\newblock {\em Nucleic acids research}, 47(D1):D506--D515, 2019.

\bibitem{Corso2024}
Gabriele Corso, Arthur Deng, Nicholas Polizzi, Regina Barzilay, and Tommi~S. Jaakkola.
\newblock Deep confident steps to new pockets: {S}trategies for docking generalization.
\newblock In {\em International Conference on Learning Representations}, 2024.

\bibitem{Corso2023}
Gabriele Corso, Hannes Stärk, Bowen Jing, Regina Barzilay, and Tommi Jaakkola.
\newblock Diffdock: {D}iffusion steps, twists, and turns for molecular docking.
\newblock In {\em International Conference on Learning Representations}, 2023.

\bibitem{Eldridge1997}
Matthew~D. Eldridge, Christopher~W. Murray, Timothy~R. Auton, Gaia~V. Paolini, and Roger~P. Mee.
\newblock Empirical scoring functions: {I. T}he development of a fast empirical scoring function to estimate the binding affinity of ligands in receptor complexes.
\newblock {\em Journal of Computer-Aided Molecular Design}, 11:425--445, 1997.

\bibitem{Friesner2004}
Richard~A. Friesner, Jay~L. Banks, Robert~B. Murphy, Thomas~A. Halgren, Jasna~J. Klicic, Daniel~T. Mainz, Matthew~P. Repasky, Eric~H. Knoll, Mee Shelley, Jason~K. Perry, et~al.
\newblock Glide: {A} new approach for rapid, accurate docking and scoring. 1. {M}ethod and assessment of docking accuracy.
\newblock {\em Journal of Medicinal Chemistry}, 47(7):1739--1749, 2004.

\bibitem{gao2025fabind+}
Kaiyuan Gao, Qizhi Pei, Gongbo Zhang, Jinhua Zhu, Kun He, and Lijun Wu.
\newblock {FABind+}: Enhancing molecular docking through improved pocket prediction and pose generation.
\newblock In {\em Proceedings of the 31st ACM SIGKDD Conference on Knowledge Discovery and Data Mining}, pages 330--341, 2025.

\bibitem{goodfellow2014generative}
Ian~J Goodfellow, Jean Pouget-Abadie, Mehdi Mirza, Bing Xu, David Warde-Farley, Sherjil Ozair, Aaron Courville, and Yoshua Bengio.
\newblock Generative adversarial nets.
\newblock In {\em Advances in Neural Information Processing Systems}, volume~27, 2014.

\bibitem{HenzlerWildman2007}
Katherine Henzler-Wildman and Dorothee Kern.
\newblock Dynamic personalities of proteins.
\newblock {\em Nature}, 450(7172):964--972, 2007.

\bibitem{hert2009quantifying}
J{\'e}r{\^o}me Hert, John~J Irwin, Christian Laggner, Michael~J Keiser, and Brian~K Shoichet.
\newblock Quantifying biogenic bias in screening libraries.
\newblock {\em Nature chemical biology}, 5(7):479--483, 2009.

\bibitem{Huang2024}
Yufei Huang, Odin Zhang, Lirong Wu, Cheng Tan, Haitao Lin, Zhangyang Gao, Siyuan Li, and Stan~Z. Li.
\newblock {Re-Dock}: {T}owards flexible and realistic molecular docking with diffusion bridge.
\newblock In {\em The Forty-First International Conference on Machine Learning}, 2024.

\bibitem{huber1992robust}
Peter~J Huber.
\newblock Robust estimation of a location parameter.
\newblock In {\em Breakthroughs in Statistics: Methodology and Distribution}, pages 492--518. Springer, 1992.

\bibitem{jang2017categorical}
Eric Jang, Shixiang Gu, and Ben Poole.
\newblock Categorical reparameterization with {G}umbel-{S}oftmax.
\newblock In {\em International Conference on Learning Representations}, 2017.

\bibitem{Jumper2021}
John Jumper, Richard Evans, Alexander Pritzel, Tim Green, Michael Figurnov, Olaf Ronneberger, Kathryn Tunyasuvunakool, Russ Bates, Augustin Žídek, Anna Potapenko, and et~al.
\newblock Highly accurate protein structure prediction with {A}lphafold.
\newblock {\em Nature}, 596(7873):583--589, 2021.

\bibitem{Krivak2018}
Radoslav Krivák and David Hoksza.
\newblock P2rank: {M}achine learning based tool for rapid and accurate prediction of ligand binding sites from protein structure.
\newblock {\em Journal of Cheminformatics}, 10:1--12, 2018.

\bibitem{lanctot2017unified}
Marc Lanctot, Vinicius Zambaldi, Audrunas Gruslys, Angeliki Lazaridou, Karl Tuyls, Julien P{\'e}rolat, David Silver, and Thore Graepel.
\newblock A unified game-theoretic approach to multiagent reinforcement learning.
\newblock In {\em Advances in Neural Information Processing Systems}, volume~30, 2017.

\bibitem{Landrum2013}
Greg Landrum.
\newblock {RDKit}: {A} software suite for cheminformatics, computational chemistry, and predictive modeling, 2013.

\bibitem{Lane2023}
Thomas~J. Lane.
\newblock Protein structure prediction has reached the single-structure frontier.
\newblock {\em Nature Methods}, 20(2):170--173, 2023.

\bibitem{Lin2022}
Zeming Lin, Halil Akin, Roshan Rao, Brian Hie, Zhongkai Zhu, Wenting Lu, Allan dos Santos~Costa, Maryam Fazel-Zarandi, Tom Sercu, Sal Candido, et~al.
\newblock Language models of protein sequences at the scale of evolution enable accurate structure prediction.
\newblock 2022.

\bibitem{liu2017forging}
Zhihai Liu, Minyi Su, Li~Han, Jie Liu, Qifan Yang, Yan Li, and Renxiao Wang.
\newblock Forging the basis for developing protein--ligand interaction scoring functions.
\newblock {\em Accounts of chemical research}, 50(2):302--309, 2017.

\bibitem{Lu2022}
Wei Lu, Qifeng Wu, Jixian Zhang, Jiahua Rao, Chengtao Li, and Shuangjia Zheng.
\newblock {TankBind}: {T}rigonometry-aware neural networks for drug-protein binding structure prediction.
\newblock In {\em Advances in Neural Information Processing Systems}, volume~35, pages 7236--7249, 2022.

\bibitem{Lu2024}
Wei Lu, Jixian Zhang, Weifeng Huang, Ziqiao Zhang, Xiangyu Jia, Zhenyu Wang, Leilei Shi, Chengtao Li, Peter~G. Wolynes, and Shuangjia Zheng.
\newblock {DynamicBind}: {P}redicting ligand-specific protein-ligand complex structure with a deep equivariant generative model.
\newblock {\em Nature Communications}, 15(1):1071, 2024.

\bibitem{luttens2025rapid}
Andreas Luttens, Israel Cabeza~de Vaca, Leonard Sparring, Jos{\'e} Brea, Ant{\'o}n~Leandro Mart{\'\i}nez, Nour~Aldin Kahlous, Dmytro~S Radchenko, Yurii~S Moroz, Mar{\'\i}a~Isabel Loza, Ulf Norinder, et~al.
\newblock Rapid traversal of vast chemical space using machine learning-guided docking screens.
\newblock {\em Nature Computational Science}, 5:301--312, 2025.

\bibitem{McNutt2021}
Andrew~T. McNutt, Paul Francoeur, Rishal Aggarwal, Tomohide Masuda, Rocco Meli, Matthew Ragoza, Jocelyn Sunseri, and David~Ryan Koes.
\newblock {GNINA} 1.0: {M}olecular docking with deep learning.
\newblock {\em Journal of Cheminformatics}, 13(1):43, 2021.

\bibitem{mohri2018foundations}
Mehryar Mohri, Afshin Rostamizadeh, and Ameet Talwalkar.
\newblock {\em Foundations of machine learning}.
\newblock MIT press, 2018.

\bibitem{monderer1996potential}
Dov Monderer and Lloyd~S Shapley.
\newblock Potential games.
\newblock {\em Games and economic behavior}, 14(1):124--143, 1996.

\bibitem{Morris1996}
Garrett~M. Morris, David~S. Goodsell, Ruth Huey, and Arthur~J. Olson.
\newblock Distributed automated docking of flexible ligands to proteins: Parallel applications of {AutoDock} 2.4.
\newblock {\em Journal of Computer-Aided Molecular Design}, 10:293--304, 1996.

\bibitem{Nocedal1999}
Jorge Nocedal and Stephen~J. Wright.
\newblock {\em Numerical Optimization}.
\newblock Springer, 1999.

\bibitem{Pei2023}
Qizhi Pei, Kaiyuan Gao, Lijun Wu, Jinhua Zhu, Yingce Xia, Shufang Xie, Tao Qin, Kun He, Tie-Yan Liu, and Rui Yan.
\newblock {FABind}: Fast and accurate protein-ligand binding.
\newblock In {\em Advances in Neural Information Processing Systems}, volume~36, 2023.

\bibitem{plainer2023diffdock}
Michael Plainer, Marcella Toth, Simon Dobers, Hannes Stark, Gabriele Corso, C{\'e}line Marquet, and Regina Barzilay.
\newblock {DiffDock-Pocket}: Diffusion for pocket-level docking with sidechain flexibility.
\newblock In {\em NeurIPS 2023 Workshop on New Frontiers of AI for Drug Discovery and Development}.

\bibitem{Qiao2024}
Zhuoran Qiao, Weili Nie, Arash Vahdat, Thomas~F. Miller~III, and Animashree Anandkumar.
\newblock State-specific protein–ligand complex structure prediction with a multiscale deep generative model.
\newblock {\em Nature Machine Intelligence}, 6(2):195--208, 2024.

\bibitem{Sahu2024}
Divya Sahu, Lokendra~Singh Rathor, Shradha~Devi Dwivedi, Kamal Shah, Nagendra~Singh Chauhan, Manju~Rawat Singh, and Deependra Singh.
\newblock A review on molecular docking as an interpretative tool for molecular targets in disease management.
\newblock {\em Assay and Drug Development Technologies}, 22(1):40--50, 2024.

\bibitem{Satorras2021}
Víctor~García Satorras, Emiel Hoogeboom, and Max Welling.
\newblock E(n) equivariant graph neural networks.
\newblock In {\em International Conference on Machine Learning}, pages 9323--9332, 2021.

\bibitem{shoham2008multiagent}
Yoav Shoham and Kevin Leyton-Brown.
\newblock {\em Multiagent systems: {A}lgorithmic, Game-Theoretic, and Logical Foundations}.
\newblock Cambridge University Press, 2008.

\bibitem{silver2016mastering}
David Silver, Aja Huang, Chris~J Maddison, Arthur Guez, Laurent Sifre, George Van Den~Driessche, Julian Schrittwieser, Ioannis Antonoglou, Veda Panneershelvam, Marc Lanctot, et~al.
\newblock Mastering the game of {Go} with deep neural networks and tree search.
\newblock {\em nature}, 529(7587):484--489, 2016.

\bibitem{Stark2022}
Hannes Stärk, Octavian Ganea, Lagnajit Pattanaik, Regina Barzilay, and Tommi Jaakkola.
\newblock Equibind: Geometric deep learning for drug binding structure prediction.
\newblock In {\em International Conference on Machine Learning}, pages 20503--20521. PMLR, 2022.

\bibitem{Trott2010}
Oleg Trott and Arthur~J. Olson.
\newblock {AutoDock Vina}: {I}mproving the speed and accuracy of docking with a new scoring function, efficient optimization, and multithreading.
\newblock {\em Journal of Computational Chemistry}, 31(2):455--461, 2010.

\bibitem{Wang2002}
Renxiao Wang, Luhua Lai, and Shaomeng Wang.
\newblock Further development and validation of empirical scoring functions for structure-based binding affinity prediction.
\newblock {\em Journal of Computer-Aided Molecular Design}, 16:11--26, 2002.

\bibitem{Yang2023}
Ling Yang, Zhilong Zhang, Yang Song, Shenda Hong, Runsheng Xu, Yue Zhao, Wentao Zhang, Bin Cui, and Ming-Hsuan Yang.
\newblock Diffusion models: A comprehensive survey of methods and applications.
\newblock {\em ACM Computing Surveys}, 56(4):1--39, 2023.

\bibitem{zhang2024packdock}
Runze Zhang, Xinyu Jiang, Duanhua Cao, Jie Yu, Mingan Chen, Zhehuan Fan, Xiangtai Kong, Jiacheng Xiong, Zimei Zhang, Wei Zhang, et~al.
\newblock {PackDock}: {A} diffusion based side chain packing model for flexible protein-ligand docking.
\newblock {\em bioRxiv}, pages 2024--01, 2024.

\bibitem{Zhang2023}
Yangtian Zhang, Huiyu Cai, Chence Shi, and Jian Tang.
\newblock E3bind: An end-to-end equivariant network for protein-ligand docking.
\newblock In {\em The Eleventh International Conference on Learning Representations}, 2023.

\bibitem{zhang2023diffpack}
Yangtian Zhang, Zuobai Zhang, Bozitao Zhong, Sanchit Misra, and Jian Tang.
\newblock Diffpack: A torsional diffusion model for autoregressive protein side-chain packing.
\newblock In {\em Advances in Neural Information Processing Systems}, volume~36, pages 48150--48172, 2023.

\bibitem{zhang2025fast}
Zizhuo Zhang, Lijun Wu, Kaiyuan Gao, Jiangchao Yao, Tao Qin, and Bo~Han.
\newblock Fast and accurate blind flexible docking.
\newblock In {\em The Thirteenth International Conference on Learning Representations}, 2025.

\bibitem{zhou2023unimol}
Gengmo Zhou, Zhifeng Gao, Qiankun Ding, Hang Zheng, Hongteng Xu, Zhewei Wei, Linfeng Zhang, and Guolin Ke.
\newblock Uni-mol: {A} universal {3D} molecular representation learning framework.
\newblock In {\em International Conference on Learning Representations}, 2023.

\bibitem{zhu2022torchdrug}
Zhaocheng Zhu, Chence Shi, Zuobai Zhang, Shengchao Liu, Minghao Xu, Xinyu Yuan, Yangtian Zhang, Junkun Chen, Huiyu Cai, Jiarui Lu, et~al.
\newblock Torchdrug: A powerful and flexible machine learning platform for drug discovery.
\newblock {\em arXiv preprint arXiv:2202.08320}, 2022.

\end{thebibliography}
\bibliographystyle{plain}

\end{document}